\documentclass{article}

\usepackage[preprint]{neurips_2021}

\usepackage[utf8]{inputenc} 
\usepackage[T1]{fontenc}    
\usepackage{hyperref}       
\usepackage{url}            
\usepackage{booktabs}       
\usepackage{amsfonts}       
\usepackage{nicefrac}       
\usepackage{microtype}      
\usepackage[usenames,dvipsnames]{xcolor}         

\usepackage{amsmath}
\usepackage{amsthm} 
\usepackage{mathtools} 
\usepackage{bbold} 
\usepackage{graphicx,float}

\setcitestyle{authoryear,round}

\newtheorem{definition}{Definition} 
\newtheorem{theorem}{Theorem}

\newtheorem{proposition}{Proposition}

\newtheorem{remark}{Remark}

\usepackage{graphicx}
\usepackage{float}
\usepackage{subfig}

\makeatletter
\newcommand{\proofpart}[2]{%
  \par
  \addvspace{\medskipamount}%
  \noindent\emph{Part #1: #2}\par\nobreak
  \addvspace{\smallskipamount}%
  \@afterheading
}
\makeatother

\DeclareMathOperator*{\argmin}{argmin}

\title{Subspace Detours Meet Gromov-Wasserstein}

\author{%
    Clément Bonet \\
    Univ. Bretagne Sud, LMBA \\ 
    F-56000 Vannes \\
    \texttt{clement.bonet@univ-ubs.fr} \\
    \And
    Nicolas Courty \\
    Univ. Bretagne Sud, IRISA \\ 
    F-56000 Vannes \\
    \texttt{nicolas.courty@irisa.fr} \\
    \And
    François Septier \\
    Univ. Bretagne Sud, LMBA \\ 
    F-56000 Vannes \\
    \texttt{francois.septier@univ-ubs.fr} \\
    \And 
    Lucas Drumetz \\
    IMT Atlantique, Lab-STICC \\ 
    F-29200 Brest \\
    \texttt{lucas.drumetz@imt-atlantique.fr}
}

\begin{document}

\maketitle

\begin{abstract}
In the context of optimal transport methods, the \emph{subspace detour} approach was recently presented by \citet{muzellec2019subspace}. It consists in building a nearly optimal transport plan in the measures space from an optimal transport plan in a wisely chosen subspace, onto which the original measures are projected. The contribution of this paper is to extend this category of methods to the Gromov-Wasserstein problem, which is a particular type of transport distance involving the inner geometry of the compared distributions. After deriving the associated formalism and properties, we also discuss a specific cost for which we can show connections with the Knothe-Rosenblatt rearrangement. We finally give an experimental illustration on a shape matching problem.
\end{abstract}

\section{Introduction}

Classical optimal transport (OT) has received lots of attention recently, in particular in Machine Learning for tasks such as generative networks \citep{arjovsky2017wasserstein} or domain adaptation \citep{courty2016optimal} to name a few. It generally relies on the Wasserstein distance, that builds an optimal coupling between distributions given their geometry. Yet, this metric lacks from potentially important properties, such as translation or rotation invariance, which can be useful when comparing shapes for instance \citep{memoli2011gromov, chowdhury2021quantized}, and cannot be used directly whenever the distributions lie in different metric spaces. In order to alleviate those problems, custom solutions have been proposed, such as  \citep{alvarez2019towards, cai2020distances}.

Apart from these works, another meaningful OT distance to tackle these problems is the Gromov-Wasserstein (GW) distance, originally proposed in \citet{memoli2007use, memoli2011gromov}. It is a distance between metric spaces and has several appealing properties such as geodesics or invariances \citep{sturm2012space}. Yet, the price to be paid lies in its computational complexity, which requires to solve a quadratic optimization problem with linear constraints. A recent line of work tends to compute approximations or relaxations of the original problem, in order to spread its use in more data intensive machine learning applications. For example, \citeauthor{peyre2016gromov} use an entropic regularization in order to iterate several Sinkhorn projections \citep{cuturi2013sinkhorn}.
A related recent method imposes coupling with low-rank constraints \citep{scetbon2021linear}.
\citeauthor{titouan2019sliced} proposed a sliced approach to approximate Gromov-Wasserstein.
\citeauthor{fatras2021minibatch} studied an estimator based on mini-batches.  In \citet{chowdhury2021quantized}, authors propose to partition the space and to solve the optimal transport problem between a subset of points, before finding a coupling between all the points.

In this work, we study the \emph{subspace detour} approach for Gromov-Wasserstein. This class of method was first proposed for the Wasserstein setting in \citet{muzellec2019subspace} and consists in choosing the optimal transport plan between projected measures on a subspace, before finding a coupling on the whole space between the original measures using disintegration.
Our main contribution is to derive the subspace detours between different subspaces and to apply it for GW costs. We derive some useful properties as well as closed-form solution between Gaussians. Interestingly enough, we also propose a separable quadratic cost for the GW problem that can be related with a triangular coupling, hence bridging the gap with Knothe-Rosenblatt (KR) rearrangements. 
Illustrations of the method are also given on a shape matching problem.

\section{Background}

In this section, we introduce all the necessary material to describe the subspace detour approach, from classical optimal transport and its connection to the Knothe-Rosenblatt rearrangement, before defining subspace optimal couplings via the gluing lemma and measure disintegration. Then, we introduce the Gromov-Wasserstein problem for which we will derive the subspace detour in the next sections.

\subsection{Classical optimal transport}

Let $\mu,\nu\in\mathcal{P}(\mathbb{R}^d)$ be two probability measures. The set of couplings between $\mu$ and $\nu$ is defined as 
\begin{equation*}
    \Pi(\mu,\nu) = \{\gamma\in\mathcal{P}(\mathbb{R}^d\times \mathbb{R}^d),\ \pi^1_\#\gamma=\mu,\ \pi^2_\#\gamma=\nu\}
\end{equation*}
where $\pi^1$ and $\pi^2$ are the projections on the first and second coordinate (\emph{i.e.} $\pi^1(x,y)=x$), and $\#$ is the push forward operator, defined such that
\begin{equation*}
    \forall A\in\mathcal{B}(\mathbb{R}^d),\ T_\#\mu(A)=\mu(T^{-1}(A)).
\end{equation*}

\paragraph{Kantorovitch problem}
There exists several types of coupling between probability measures and a non exhaustive list can be found in \citep{villani2008optimal}[Chapter 1]. Among them, the so called optimal coupling is the minimizer of the following Kantorovitch problem:
\begin{equation} \label{kantorovitch}
    \inf_{\gamma\in\Pi(\mu,\nu)}\ \int c(x,y)\mathrm{d}\gamma(x,y)
\end{equation}
with $c$ some cost function. When $c(x,y)=\|x-y\|_2^2$, then it defines the Wasserstein distance
\begin{equation} \label{wasserstein}
    W_2^2(\mu,\nu) = \inf_{\gamma\in\Pi(\mu,\nu)}\ \int \|x-y\|_2^2\ \mathrm{d}\gamma(x,y).
\end{equation}
The Kantorovitch problem \eqref{kantorovitch} is known to admit a solution when $c$ is nonnegative and lower semi-continuous \citep{santambrogio2015optimal}[Theorem 1.7]. When the optimal coupling is of the form $\gamma=(Id,T)_\#\mu$ with $T$ some deterministic map such that $T_\#\mu=\nu$, $T$ is called the Monge map.

In one dimension, with $\mu$ atomless, the solution to \eqref{wasserstein} is a deterministic coupling of the form \citep{santambrogio2015optimal}[Theorem 2.5]
\begin{equation} \label{increasing_rearragnement}
    T=F_\nu^{-1}\circ F_\mu,
\end{equation}
where $F_\mu$ is the cumulative distribution function of $\mu$ and $F_\nu^{-1}$ the quantile function of $\nu$.
This map is also known as the increasing rearrangement.

\paragraph{Knothe-Rosenblatt rearrangement} \label{knothe}

Another interesting coupling is the Knothe-Rosenblatt rearrangement, which takes advantage of the increasing rearrangement in one dimension by iterating over the dimension and disintegrating. Concatenating all the increasing rearrangements between the conditional probabilities, we obtain the KR rearrangement, which turns out to be a nondecreasing triangular map (\emph{i.e.} $T:\mathbb{R}^d\to\mathbb{R}^d$, for all $x\in\mathbb{R}^d$, $T(x)=(T_1(x_1),\dots,T_j(x_1,\dots,x_j),\dots,T_d(x))$ and for all $j$, $T_j$ is nondecreasing with respect to $x_j$), and a deterministic coupling (\emph{i.e.} $T_\#\mu=\nu$) \citep{villani2008optimal, santambrogio2015optimal, jaini2019sum}.

\citeauthor{carlier2010knothe} made a connection between this coupling and optimal transport by showing that it can be obtained as the limit of optimal transport plans for a degenerated cost 
\begin{equation*} \label{eq:degenerated_cost}
    c_t(x,y)=\sum_{i=1}^d\lambda_i(t)(x_i-y_i)^2,
\end{equation*}
where for all $i\in\{1,\dots,d\}$, $t>0$, $\lambda_i(t)>0$ and for all $i\ge 2$, $\frac{\lambda_i(t)}{\lambda_{i-1}(t)}\xrightarrow[t\to 0]{}0$. This cost can be recast as in \citep{bonnotte2013unidimensional} as $c_t(x,y)=(x-y)^T A_t (x-y)$ where $A_t = \mathrm{diag}(\lambda_1(t),\dots,\lambda_d(t))$.
This formalizes into the following Theorem:
\begin{theorem}[\cite{carlier2010knothe, santambrogio2015optimal}] \label{th_KnotheToBrenier}
    Let $\mu$ and $\nu$ be two absolutely continuous measures on $\mathbb{R}^d$, with compact supports. Let $\gamma_t$ be an optimal transport plan for the cost $c_t$, let $T_K$ be the Knothe-Rosenblatt map between $\mu$ and $\nu$, and $\gamma_K = (Id\times T_K)_\#\mu$ the associated transport plan. Then, we have $\gamma_t\xrightarrow[t\rightarrow 0]{\mathcal{D}}\gamma_K$. Moreover, if $\gamma_t$ are induced by transport maps $T_t$, then $T_t$ converges in $L^2(\mu)$ when t tends to zero to the Knothe-Rosenblatt rearrangement.
\end{theorem}

\subsection{Subspace detours and disintegration} \label{subspace_detour}

\citeauthor{muzellec2019subspace} proposed another OT problem by optimizing over the couplings which share a measure on a subspace. More precisely, they defined subspace optimal plans for which the shared measure is the OT plan between projected measures.
\begin{definition}[Subspace-Optimal Plans \citep{muzellec2019subspace} Definition 1]
    Let $\mu,\nu\in\mathcal{P}_2(\mathbb{R}^d)$ and let $E\subset \mathbb{R}^d$ be a $k$-dimensional subspace. Let $\gamma_{E}^*$ be an OT plan between $\mu_E=\pi^E_\#\mu$ and $\nu_E=\pi^E_\#\nu$ (with $\pi^E$ the orthogonal projection on $E$). Then the set of $E$-optimal plans between $\mu$ and $\nu$ is defined as $\Pi_E(\mu,\nu)=\{\gamma\in\Pi(\mu,\nu)|\ (\pi^E,\pi^E)_\#\gamma=\gamma_{E}^*\}$.
\end{definition}

By the Gluing lemma \citep{villani2008optimal}, it is possible to construct a coupling $\gamma\in\Pi(\mu,\nu)$ such that $(\pi^E,\pi^E)_\#\gamma=\gamma_{E}^*$. A way to do that is to rely on disintegration.

\paragraph{Disintegration}
Let $(Y,\mathcal{Y})$ and $(Z,\mathcal{Z})$ be measurable spaces, and $(X,\mathcal{X})=(Y\times Z,\mathcal{Y}\otimes\mathcal{Z})$ the product measurable space. Then, for $\mu\in\mathcal{P}(X)$, we denote $\mu_Y = \pi^Y_\#\mu$ and $\mu_Z=\pi^Z_\#\mu$ the marginals, where $\pi^Y$ (respectively $\pi^Z$) is the projection on $Y$ (respectively Z). Then, a family $(K(y,\cdot))_{y\in\mathcal{Y}}$ is a disintegration of $\mu$ if for all $y\in Y$, $K(y,\cdot)$ is a measure on $Z$, for all $A\in\mathcal{Z}$, $K(\cdot,A)$ is measurable and
\begin{equation*}
    \forall \phi\in C(X),\ \int_{Y\times Z} \phi(y,z)\mathrm{d}\mu(y,z) = \int_Y\int_Z \phi(y,z)K(y,\mathrm{d}z)\mathrm{d}\mu_Y(y),
\end{equation*}
where $C(X)$ is the set of continuous functions on $X$. We can note $\mu=\mu_Y\otimes K$. $K$ is a probability kernel if for all $y\in Y$, $K(y,Z)=1$. The disintegration of a measure actually corresponds to conditional laws in the context of probabilities. This concept will allow us to obtain measures on the whole space from marginals on subspaces.

In the case where $X=\mathbb{R}^d$, which is the main case of interest in the remainder of the paper, we have existence and uniqueness of the disintegration (see Box 2.2 of \citet{santambrogio2015optimal} or Chapter 5 of \citet{ambrosio2008gradient} for the more general case).

\paragraph{Coupling on the whole set}

Let's note $\mu_{E^\bot|E}$ and $\nu_{E^\bot|E}$ the disintegrated measures on the orthogonal spaces (\emph{i.e.} $\mu=\mu_E \otimes \mu_{E^\bot|E}$ and $\nu=\nu_E\otimes \nu_{E^\bot|E}$). Then, to get a transport plan between the two originals measures on the whole space, we can look for another coupling between disintegrated measures $\mu_{E^\bot|E}$ and $\nu_{E^\bot|E}$. In particular, two such couplings are proposed in \citet{muzellec2019subspace}, the Monge-Independent (MI) plan
\begin{equation*}
    \pi_{\mathrm{MI}} = \gamma_{E}^*\otimes (\mu_{E^\bot|E}\otimes \nu_{E^\bot |E})
\end{equation*}
where we take the independent coupling between $\mu_{E^\bot|E}(x_E,\cdot)$ and $\nu_{E^\bot|E}(y_E,\cdot)$ for $\gamma_{E}^*$ almost every $(x_E,y_E)$, and the Monge-Knothe (MK) plan
\begin{equation*}
    \pi_{\mathrm{MK}} = \gamma_{E}^*\otimes \gamma_{E^\bot|E}^*
\end{equation*}
where $\gamma_{E^\bot |E}^*((x_E,y_E),\cdot)$ is an optimal plan between $\mu_{E^\bot|E}(x_E,\cdot)$ and $\nu_{E^\bot|E}(y_E,\cdot)$ for $\gamma_{E}^*$ almost every $(x_E,y_E)$.
\citeauthor{muzellec2019subspace} observed that MI is more adapted to noisy environments since it only computes the OT plan on the subspace. MK is more suited for applications where we want to prioritize some subspace but where all the directions still contain relevant informations.

\subsection{Gromov-Wasserstein}
 Formally, the Gromov-Wasserstein distance allows to compare metric measure spaces (mm-space), triplets $(X,d_X,\mu_X)$ and $(Y,d_Y,\mu_Y)$ where $(X,d_X)$ and $(Y,d_Y)$ are complete separable metric spaces and $\mu_X$, $\mu_Y$ Borel probability measures on $X$ and $Y$ \citep{sturm2012space}, by computing
\begin{equation*}
    GW(X,Y) = \inf_{\gamma\in\Pi(\mu_X,\mu_Y)}\iint L(d_X(x,x'),d_Y(y,y')) \mathrm{d}\gamma(x,y)\mathrm{d}\gamma(x',y')
\end{equation*}
where $L$ is some loss on $\mathbb{R}$. It has actually been extended to other spaces by replacing the distances by cost functions $c_X$ and $c_Y$, as \emph{e.g.} in \citep{chowdhury2019gromov}. Furthermore, it has many appealing properties such as having invariances (which depend on the costs).

\citeauthor{vayer2020contribution} studied notably this problem in the setting where $X$ and $Y$ are Euclidean spaces, with $L(x,y)=(x-y)^2$ and $c(x,x')=\langle x,x'\rangle$ or $c(x,x')=\|x-x'\|_2^2$. In particular, let $\mu\in\mathcal{P}(\mathbb{R}^p)$, $\nu\in\mathcal{P}(\mathbb{R}^q)$, the inner-GW problem is defined as
\begin{equation} \label{gw_ip}
    \mathrm{InnerGW}(\mu,\nu) = \inf_{\gamma\in\Pi(\mu,\nu)} \iint(\langle x,x'\rangle_p-\langle y,y'\rangle_q)^2\ \mathrm{d}\gamma(x,y)\mathrm{d}\gamma(x',y').
\end{equation}
For this problem, a closed-form in one dimension can be found:
\begin{theorem}[\cite{vayer2020contribution} Theorem 4.2.4] \label{gw_ip_1d}
    Let $\mu,\nu\in\mathcal{P}(\mathbb{R})$, with $\mu$ absolutely continuous with respect to the Lebesgue measure. Let $F_\mu^\nearrow(x)\coloneqq F_\mu(x) =\mu(]-\infty,x])$ be the cumulative distribution function and $F_\mu^\searrow(x)=\mu(]-x,+\infty[)$ the anti-cumulative distribution function. Let $T_{asc}(x)=F_\nu^{-1}(F_\mu^\nearrow(x))$ and $T_{desc}(x)=F_\nu^{-1}(F_\mu^\searrow(-x))$. Then, an optimal solution of \eqref{gw_ip} is achieved either by $\gamma=(Id\times T_{asc})_\#\mu$ or by $\gamma=(Id\times T_{desc})_\#\mu$.
\end{theorem}

\section{Subspace detours for $GW$}

In this section, we propose to extend subspace detours from \citet{muzellec2019subspace} with Gromov-Wasserstein costs. We show that we can even take subspaces of different dimensions, and still obtain a coupling on the whole space using the Independent or the Monge-Knothe coupling. Then, we derive some properties analogously to \citet{muzellec2019subspace}, as well as some closed-form solutions between Gaussians.

\subsection{Motivations}

First, we adapt the definition of subspace optimal plans for difference subspaces. Indeed, the Gromov-Wasserstein distance being able to compare data on spaces of different dimensions, we can argue that the main information would not be in the same subspace for both datasets. For example, by rotating a dataset, we would change the subspace of interest and most likely lose information as we can see on Figure \ref{fig:GW}. 
On this illustration, we use as a source one moon of the Two moons dataset, and obtain a target by rotating it by an angle of $\frac{\pi}{2}$. As GW with $c(x,x')=\|x-x'\|_2^2$ is invariant with respect to isometries, we are able to recover the exact correspondence between the points. However, when choosing a subspace to project both the source and target, we completely lose the optimal coupling between them. Nonetheless, by choosing more wisely one subspace by dataset (using here the first component of the principal component analysis (PCA) decomposition), we find the right coupling. This illustration underlines the idea that the choice of both subspaces is important. A way of choosing the subspaces could be to project on the subspace containing the more information for each dataset using \emph{e.g.} PCA independently on each distribution. \citeauthor{muzellec2019subspace} proposed to optimize the optimal transport cost with respect to an orthonormal matrix with a projected gradient descent, which could be extended to an optimization over two orthonormal matrices in our context.

\begin{figure}[t]
    \centering
    \includegraphics[scale=0.35]{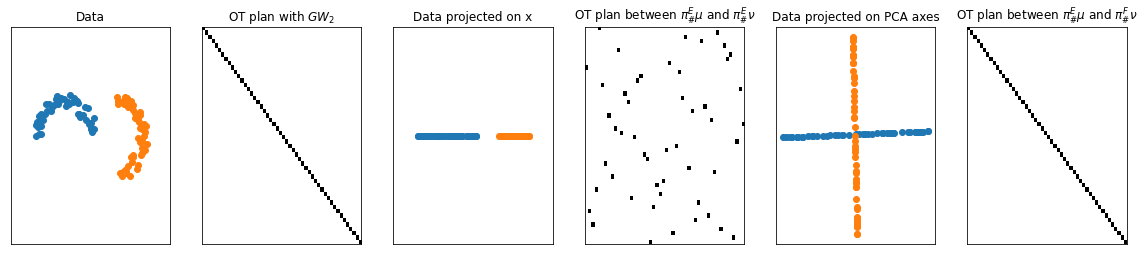}
    \caption{From left to right: Data (moons), OT plan obtained with GW for $c(x,x')=\|x-x'\|_2^2$, Data projected on the 1st axis, OT plan obtained between the projected measures, Data projected on their 1st PCA component, OT plan obtained between the the projected measures}
    \label{fig:GW}
\end{figure}

By allowing to have different subspaces, we get the following definition of subspace optimal plans.

\begin{definition}
    Let $\mu\in\mathcal{P}_2(\mathbb{R}^p)$, $\nu\in\mathcal{P}_2(\mathbb{R}^q)$, $E$ be a $k$-dimensional subspace of $\mathbb{R}^p$ and $F$ a $k'$-dimensional subspace of $\mathbb{R}^q$. Let $\gamma_{E,F}^*$ be an optimal transport plan for $GW$ between $\mu_E=\pi^E_\#\mu$ and $\nu_F=\pi^F_\#\nu$ (with $\pi^E$ (resp. $\pi^F$) the orthogonal projection on $E$ (resp. $F$)). Then the set of $(E,F)$-optimal plans between $\mu$ and $\nu$ is defined as $\Pi_{E,F}(\mu,\nu)=\{\gamma\in\Pi(\mu,\nu)|\ (\pi^E,\pi^F)_\#\gamma=\gamma_{E,F}^*\}$.
\end{definition}

Analogously to \citet{muzellec2019subspace} (Section \ref{subspace_detour}), we can obtain from $\gamma_{E,F}^*$ a coupling on the whole set by either defining the Monge-Independent plan $\pi_{\mathrm{MI}} = \gamma_{E,F}^*\otimes (\mu_{E^\bot|E}\otimes \nu_{F^\bot |F})$ or the Monge-Knothe plan $\pi_{\mathrm{MK}} = \gamma_{E,F}^*\otimes \gamma_{E^\bot\times F^\bot|E\times F}^*$ where OT plans are taken with some OT cost such as \emph{e.g.} $GW$.

\subsection{Properties}

Following \citet{muzellec2019subspace}, the Monge-Knothe coupling is the optimal measure among the subspace optimal plans for the corresponding cost. We show it for the Gromov-Wasserstein distance with cost $L$, which is a direct transposition of Proposition 1 in \citet{muzellec2019subspace}.

\begin{proposition} \label{prop1}
    Let $\mu\in\mathcal{P}(\mathbb{R}^p)$ and $\nu\in\mathcal{P}(\mathbb{R}^q)$, $E\subset \mathbb{R}^p$, $F\subset\mathbb{R}^q$, $\pi_{\mathrm{MK}}=\gamma_{E,F}^*\otimes \gamma_{E^\bot\times F^\bot|E\times F}^*$ where $\gamma^*$ are optimal for the Gromov-Wasserstein problem with cost $L$. Then we have:
    \begin{equation*}
        \pi_{\mathrm{MK}}\in\argmin_{\gamma\in\Pi_{E,F}(\mu,\nu)} \iint L(x,x',y,y') \mathrm{d}\gamma(x,y)\mathrm{d}\gamma(x',y').
    \end{equation*}
\end{proposition}

Key properties of $GW$ that we would like to keep are its invariances. We show in two particular cases that we conserve them on the orthogonal spaces (since the measure on $E\times F$ is fixed).

\begin{proposition} \label{prop2}
    Let $\mu\in\mathcal{P}(\mathbb{R}^p)$, $\nu\in\mathcal{P}(\mathbb{R}^q)$, and denote 
    \begin{equation*}
        GW_{E,F}(\mu,\nu) = \inf_{\gamma\in\Pi_{E,F}(\mu,\nu)}\ \iint L(x,x',y,y')\mathrm{d}\gamma(x,y)\mathrm{d}\gamma(x',y').
    \end{equation*}
    For $L(x,x',y,y')=\big(\|x-x'\|_2^2 - \|y-y'\|_2^2\big)^2$, $GW_{E,F}$ is invariant with respect to translations and isometries on $E^\bot$ and $F^\bot$.
    
    For $L(x,x',y,y') = \big(\langle x,x'\rangle_p - \langle y,y'\rangle_q\big)^2$, $GW_{E,F}$ is invariant with respect to isometries on $E^\bot$ and $F^\bot$.
\end{proposition}

We refer to Appendix \ref{proof_props} for the proofs of the two previous propositions.

\subsection{Closed-form between Gaussians}

We can also derive explicit formulas between Gaussians in particular cases. Let $q\le p$, $\mu=\mathcal{N}(m_\mu,\Sigma)\in\mathcal{P}(\mathbb{R}^p)$, $\nu=\mathcal{N}(m_\nu,\Lambda)\in\mathcal{P}(\mathbb{R}^q)$ two Gaussian measures with $\Sigma=P_\mu D_\mu P_\mu^T$ and $\Lambda=P_\nu D_\nu P_\nu^T$. As previously, let $E\subset\mathbb{R}^p$ and $F\subset\mathbb{R}^q$ be respectively $k$ and $k'$ dimensional subspaces.
Following \citet{muzellec2019subspace}, we represent $\Sigma$ in an orthonormal basis of $E\oplus E^\bot$, and $\Lambda$ in an orthonormal basis of $F\oplus F^\bot$, \emph{i.e.} $\Sigma = \begin{pmatrix}\Sigma_E & \Sigma_{EE^\bot} \\ \Sigma_{E^\bot E} & \Sigma_{E^\bot}\end{pmatrix}$. Now, let's denote 
\begin{equation*}
    \Sigma/\Sigma_E = \Sigma_{E^\bot}-\Sigma_{EE^\bot}^T\Sigma_E^{-1}\Sigma_{EE^\bot}    
\end{equation*}
the Schur complement of $\Sigma$ with respect to $\Sigma_E$. We know that the conditionals of Gaussians are Gaussians, and of covariance the Schur complement (see \emph{e.g.} \citet{rasmussen2003gaussian,von2014mathematical}).

For $L(x,x',y,y')=\big(\|x-x'\|_2^2-\|y-y'\|_2^2\big)^2$, we have for now no certainty that the optimal transport plan is Gaussian. By restricting the minimization problem to Gaussian couplings, \citeauthor{salmona2021gromov} showed that there is a solution $\gamma^*=(Id,T)_\#\mu\in\Pi(\mu,\nu)$ with $\mu=\mathcal{N}(m_\mu,\Sigma)$, $\nu=\mathcal{N}(m_\nu,\Lambda)$ and 
\begin{equation} \label{monge_map_gw}
    \forall x\in\mathbb{R}^d,\ T(x)=m_\nu+P_\nu A P_\mu^T (x-m_\mu)
\end{equation}
where $A=\begin{pmatrix} \Tilde{I}_q D_\nu^{\frac12}(D_\mu^{(q)})^{-\frac12} & 0_{q,p-q} \end{pmatrix}\in\mathbb{R}^{q\times p}$ and $\Tilde{I}_q$ is of the form $\mathrm{diag}\big((\pm 1)_{i\le q}\big)$.

By combining the results of \citet{muzellec2019subspace} and \citet{salmona2021gromov}, we get the following closed-form for Monge-Knothe couplings.

\begin{proposition}
    Suppose $p\ge q$ and $k=k'$. For the Gaussian restricted GW problem, a Monge-Knothe transport map between $\mu=\mathcal{N}(m_\mu,\Sigma)\in\mathcal{P}(\mathbb{R}^p)$ and $\nu=\mathcal{N}(m_\nu,\Lambda)\in\mathcal{P}(\mathbb{R}^q)$ is, for all $x\in\mathbb{R}^p$, $T_{\mathrm{MK}}(x) = m_\nu + B(x-m_\mu)$ where
    \begin{equation*}
        B = \begin{pmatrix}
        T_{E,F} & 0 \\
        C & T_{E^\bot,F^\bot|E,F}
    \end{pmatrix}
    \end{equation*}
    with $T_{E,F}$ an optimal transport map between $\mathcal{N}(0_E,\Sigma_E)$ and $\mathcal{N}(0_F,\Lambda_F)$ (of the form \eqref{monge_map_gw}), $T_{E^\bot,F^\bot|E,F}$ an optimal transport map between $\mathcal{N}(0_{E^\bot},\Sigma/\Sigma_E)$ and $\mathcal{N}(0_{F^\bot},\Lambda/\Lambda_F)$ and $C$ satisfying
    \begin{equation*}
        C = (\Lambda_{F^\bot F} (T_{E,F}^T)^{-1} - T_{E^\bot,F^\bot|E,F}\Sigma_{E^\bot E})\Sigma_E^{-1}.
    \end{equation*}
\end{proposition}

\begin{proof}
    See Appendix \ref{quadraticGW_Gaussians}.
\end{proof}


Suppose that $k\ge k'$, $m_\mu=0$, $m_\nu=0$ and let $T_{E,F}$ be an optimal transport map between $\mu_E$ and $\nu_F$ (of the form \eqref{monge_map_gw}). We can derive a formula for the Monge-Independent coupling for the inner-GW problem and the Gaussian restricted GW problem.
\begin{proposition}
    $\pi_{\mathrm{MI}}=\mathcal{N}(0_{p+q},\Gamma)$ where $\Gamma = \begin{pmatrix} \Sigma & C \\ C^T & \Lambda \end{pmatrix}$ with
    \begin{equation*}
        C = (V_E\Sigma_E+V_{E^\bot}\Sigma_{E^\bot E})T_{E,F}^T(V_F^T+\Lambda_F^{-1}\Lambda_{F^\bot F}^T V_{F^\bot}^T)
    \end{equation*}
    where $T_{E,F}$ is an optimal transport map, either for the inner-GW problem or the Gaussian restricted problem.
\end{proposition}

\begin{proof}
    See Appendix \ref{mi_gaussians}.
\end{proof}

\subsection{Limit of optimal transport plans?}

Another interesting property derived in \citet{muzellec2019subspace} of the Monge-Knothe coupling is that it can be obtained as the limit of classic optimal transport plans, similar to Theorem \ref{th_KnotheToBrenier}, using a separable cost of the form
\begin{equation*}
    c_t(x,y) = (x-y)^T P_t (x-y)
\end{equation*}
with $P_t = V_EV_E^T + tV_{E^\bot}V_{E^\bot}^T$ and $(V_E,V_{E^\bot})$ an orthonormal basis of $\mathbb{R}^p$. 

However, this property is not valid for the classical Gromov-Wasserstein cost (\emph{e.g.} $L(x,x',y,y')=(d_X(x,x')^2-d_Y(y,y')^2)^2$ or $L(x,x',y,y')=(\langle x,x'\rangle_p-\langle y,y'\rangle_q)^2$) as the cost is not separable. Motivated by this question, we ask ourselves in the following if we can derive a quadratic optimal transport cost for which we would have this property.

\paragraph{Construction and properties of the Hadamard-Wasserstein problem}

The main idea of the proof of Theorem \ref{th_KnotheToBrenier} in \cite{carlier2010knothe} is to decompose the objective function as
\begin{equation*}
    \int c_t(x,y)\mathrm{d}\gamma(x,y) =\lambda_1(t)\Big( \int (x_1-y_1)^2 \mathrm{d}\gamma(x,y) + \int \sum_{k=2}^{d} \frac{\lambda_k(t)}{\lambda_{1}(t)} (x_k-y_k)^2 \mathrm{d}\gamma(x,y)\Big),
\end{equation*}
before taking the limit $t\to 0$ which makes the right-hand term vanish and allows to conclude on the limit of the first marginal of the optimal map. Reasoning by induction on the dimension, \citeauthor{carlier2010knothe} are able to deal with one term at a time, and finally show that the limit of the optimal map is the Knothe-Rosenblatt transport (\ref{knothe}). Another key ingredient is to have access to a unique transport map between measures in $\mathbb{R}$, as it is the case for the Wasserstein distance with cost $c(x,y)=\frac12 (x-y)^2$, the Monge map being the increasing rearrangement \eqref{increasing_rearragnement} (it can actually be extended to smoothly strictly convex costs, see \cite{santambrogio2015optimal}[Theorem 2.9]). 

For now, the only cost for which we have an optimal transport map in 1D is for the inner product \citep{vayer2020contribution}. Hence, we need a cost which reduces to inner-GW \eqref{gw_ip} in 1D. A natural choice is therefore to use the following cost:
\begin{equation} \label{loss_HW}
    \forall x,x',y,y' \in\mathbb{R}^d,\ L(x,x',y,y') = \sum_{k=1}^d (x_kx_k'-y_ky_k')^2 = \|x\odot x' - y\odot y'\|_2^2
\end{equation}
as a loss function, where $\odot$ is the Hadamard product (element wise product). We define the following ``Hadamard  Wasserstein'' problem
\begin{equation} \label{HadamardWasserstein}
    \mathcal{HW}(\mu,\nu) = \inf_{\gamma\in\Pi(\mu,\nu)} \iint \|x\odot x'-y\odot y'\|_2^2\ \mathrm{d}\gamma(x,y)\mathrm{d}\gamma(x',y').
\end{equation}

\paragraph{Properties}
The loss $L$ \eqref{loss_HW} satisfies well the separability condition and reduces to the inner-GW loss in 1D. We can therefore define a degenerated version of it,
\begin{equation} \label{degenerated_cost}
    \begin{aligned}
        \forall x,x',y,y',\ L_t(x,x',y,y') &= \sum_{k=1}^d \Big( \prod_{i=1}^{k-1} \lambda_t^{(i)}\Big) (x_kx_k'-y_ky_k')^2 \\
        &= (x\odot x'-y\odot y')A_t(x\odot x'-y\odot y')
    \end{aligned}
\end{equation}
with $A_t=\mathrm{diag}(1,\lambda_t^{(1)},\lambda_t^{(1)}\lambda_t^{(2)},...,\prod_{i=1}^{d-1}\lambda_t^{(i)})$, and such as for all $t>0$, and for all $i\in\{1,\dots,d-1\}$, $\lambda_t^{(i)}>0$ and $\lambda_t^{(i)}\xrightarrow[t\to 0]{}0$. We denote $\mathcal{HW}_t$ the problem \eqref{HadamardWasserstein} with the degenerate cost \eqref{degenerated_cost}. We will derive some useful properties which are usual for the regular Gromov-Wasserstein cost.

\begin{proposition} \label{prop_hadamard}
    Let $\mu,\nu\in\mathcal{P}(\mathbb{R}^d)$.
    
    \begin{enumerate}
        \item The problem \eqref{HadamardWasserstein} always admits a minimizer.
        \item $\mathcal{HW}$ is a pseudometric (\emph{i.e.} it is symmetric, nonnegative, $\mathcal{HW}(\mu,\mu)=0$ and it satisfies the triangle inequality).
        \item $\mathcal{HW}$ is invariant to reflexion with respect to axes.
    \end{enumerate}
\end{proposition}

\begin{proof}
    See Appendix \ref{proof_Hadamard}.
\end{proof}

$\mathcal{HW}$ loses somes properties compared to $GW$. Indeed, it is only invariant with respect to axes and it can compare only measures lying in the same Euclidean space in order for the distance to be well defined. Nonetheless, we show in the following that we can derive some links with triangular couplings in the same way as the Wasserstein distance and KR.

We first define a triangular coupling different from the Knothe-Rosenblatt rearrangement in the sense that each map will not be nondecreasing. Indeed, following Theorem \ref{gw_ip_1d}, the solution of each 1D problem 
\begin{equation*}
    \argmin_{\gamma\in\Pi(\mu,\nu)}\ \iint (xx'-yy')^2\ \mathrm{d}\gamma(x,y)\mathrm{d}\gamma(x',y')
\end{equation*}
is either $(Id\times T_{\mathrm{asc}})_\#\mu$ or $(Id\times T_{\mathrm{desc}})_\#\mu$. Hence, at each step $k\ge 1$, if we disintegrate the joint law of the $k$ first variables as $\mu^{1:k}=\mu^{1:k-1}\otimes \mu^{k|1:k-1}$, the optimal transport map $T(\cdot|x_1,\dots,x_k)$ will be the solution of
\begin{equation*}
    \argmin_{T\in\{T_{\mathrm{asc}}, T_{\mathrm{desc}}\}}\ \iint \big(x_kx_k'-T(x_k)T(x_k')\big)^2\ \mu^{k\mid1:k-1}(\mathrm{d}x_k\mid x_{1:k-1})\mu^{k\mid1:k-1}(\mathrm{d}x_k'\mid x_{1:k-1}').
\end{equation*}

We now state the main theorem where we show that the limit of the OT plans obtained with the degenerated cost will be the triangular coupling we just defined.

\begin{theorem} \label{ThGWKR}
    Let $\mu$ and $\nu$ be two absolutely continuous measures on $\mathbb{R}^d$ such that $\int\|x\|_2^4\ \mu(\mathrm{d}x) < +\infty$, $\int \|y\|_2^4\ \nu(\mathrm{d}y)<+\infty$ and with compact support. Let $\gamma_t$ be an optimal transport plan for $\mathcal{HW}_t$, let $T_K$ be the alternate Knothe-Rosenblatt map between $\mu$ and $\nu$ as defined in the last paragraph, and let $\gamma_K=(Id\times T_K)_\#\mu$ be the associated transport plan. Then, we have $\gamma_t\xrightarrow[t\to 0]{\mathcal{D}}\gamma_K$. Moreover, if $\gamma_t$ are induced by transport maps $T_t$, then $T_t\xrightarrow[t\to 0]{L^2(\mu)} T_K$.
\end{theorem}

\begin{proof}
    See appendix \ref{ProofThGWKR}.
\end{proof}

We report in Appendix \ref{appendix:HW} how to compute $\mathcal{HW}$ \eqref{HadamardWasserstein} in the discrete setting.

\section{Illustrations}

We use the Python Optimal Transport (POT) library \citep{flamary2021pot} to compute the different optimal transport problems involved in this illustration. We are interested here in solving a 3D mesh registration problem, which is a natural application of Gromov-Wasserstein~\citep{memoli2011gromov} since it enjoys invariances with respect to isometries such as permutations, and can also naturally exploit the topology of the meshes. For this purpose, we selected two base meshes  from the {\sc Faust} dataset~\citep{bogo2014faust}, which provides ground
truth correspondences between shapes. The information available from those meshes are geometrical ($6890$ vertices positions) and topological (mesh connectivity). These two meshes are represented, along with the visual results of the registration, in Figure~\ref{fig:mesh}. In order to visually depict the quality of the assignment induced by the transport map, we propagate through it a color code of the source vertices toward their associated counterpart vertices in the target mesh. Both original color coded source and associated target ground truth are available on the first line of the illustration. To compute our method, we simply use as a natural subspace for both meshes the algebraic connectivity of the mesh topological information, also known as the Fiedler vector (eigenvector associated to the second smallest eigenvalue of the un-normalized Laplacian matrix). Reduced to a 1D optimal transport problem, following Eq.~\ref{gw_ip}
, the computation time is very low ($\sim5$secs. on a standard laptop), and the associated matching is very good with more than $98\%$ of correct assignments. We qualitatively compare this result to Gromov-Wasserstein mappings induced by different cost functions, in the second line of Figure~\ref{fig:mesh}: adjacency \citep{xu2019scalable}, weighted adjacency (weights are given by distances between vertices), heat kernel (derived from the un-normalized Laplacian) \citep{chowdhury2021generalized} and finally geodesic distances over the meshes. In average, computing the Gromov-Wasserstein mapping using POT took around $10$ minutes of time. Both methods based on adjacency fail to recover a meaningful mapping. Heat kernel allows to map continuous areas of the source mesh, but fails in recovering a global structure. Finally, the geodesic distance gives a much more coherent mapping, but has inverted left and right of the human figure. Notably, a significant extra computation time was induced by the computation of the geodesic distances ($\sim1$h/mesh using the NetworkX~\citep{hagberg2008exploring} shortest path procedure). As a conclusion, and despite the simplification of the original problem, our method performs best, with a speed-up of two-orders of magnitude. 

\begin{figure}[h] 
    \centering
    \includegraphics[width=0.92\linewidth]{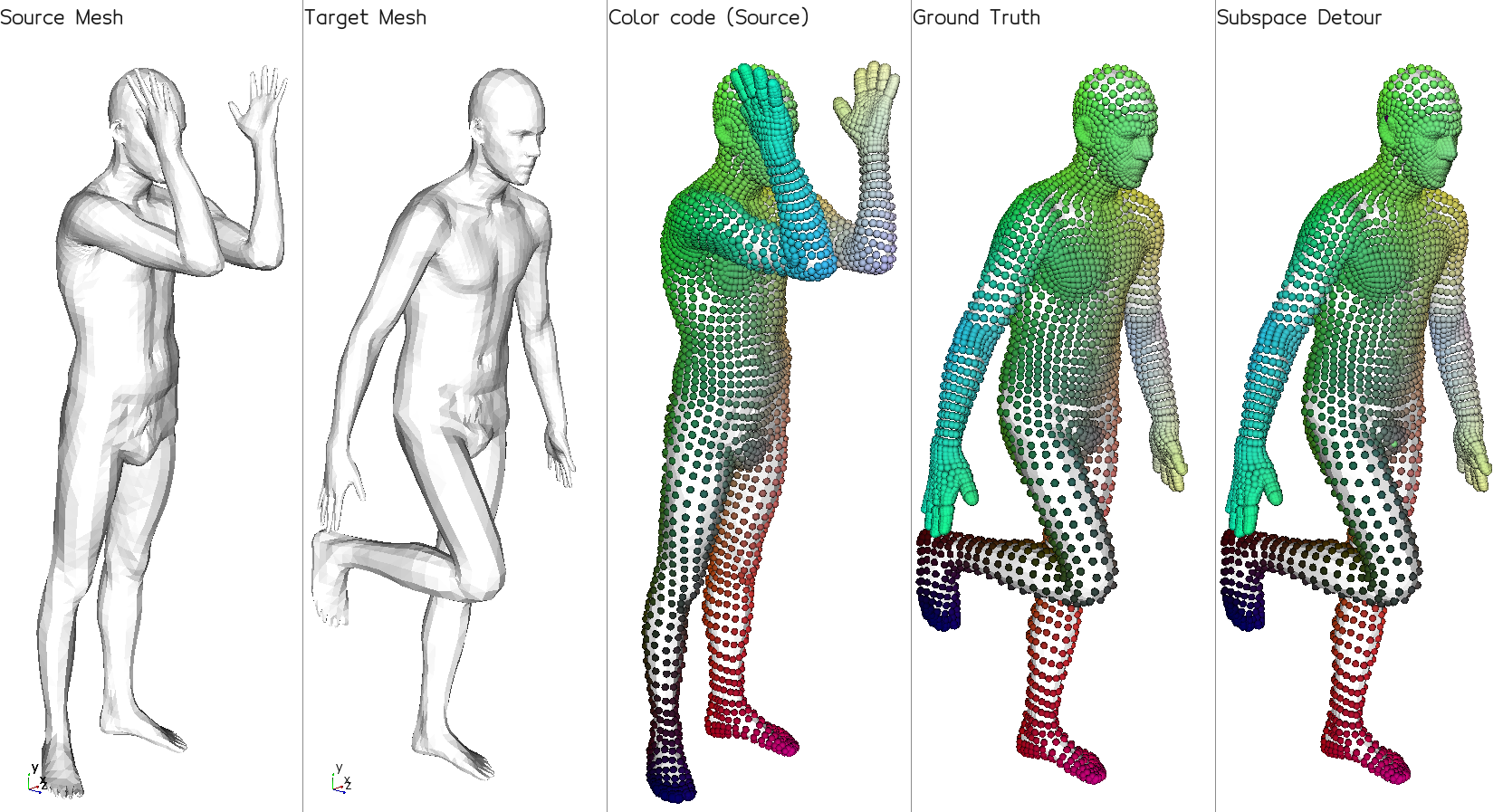}
    \includegraphics[width=0.92\linewidth]{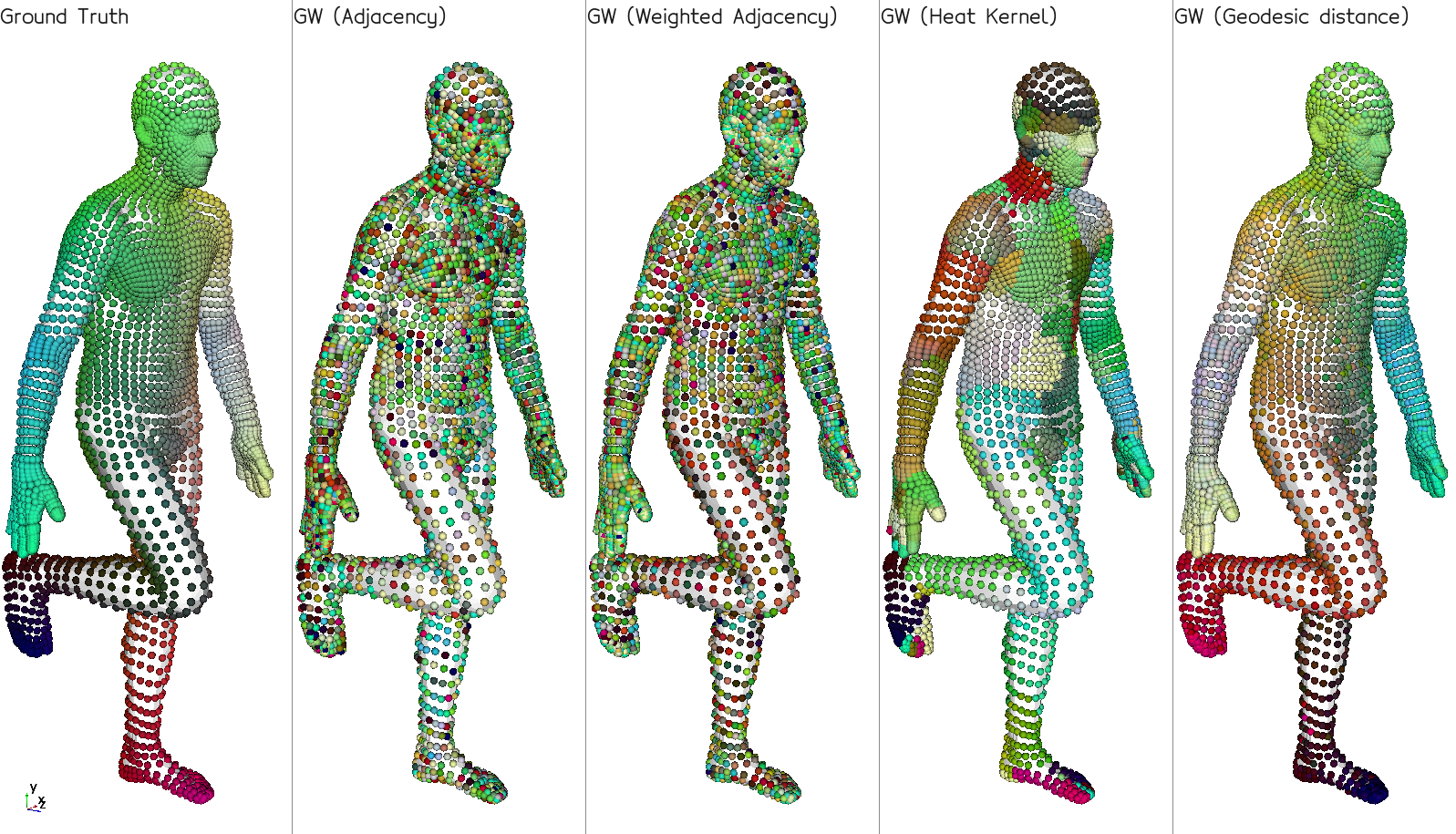}
    \caption{3D Mesh registration. (First row) source and target meshes, color code of the source, ground truth color code on the target, result of subspace detour using Fiedler vectors as subspace. (Second row) After recalling the expected ground truth for ease of comparison, we present results of different Gromov-Wasserstein mappings obtained with metrics based on adjacency, heat kernel and geodesic distances.}
    \label{fig:mesh}
\end{figure}

\section{Discussion}

We proposed in this work to extend the subspace detour approach to different subspaces, and to other optimal transport costs such as Gromov-Wasserstein. Being able to project on different subspaces can be useful when the data are not aligned and do not share the same axes of interest, as well as when we are working between different metric spaces as it is the case for example with graphs. However, a question arising is how to choose these subspaces. Since the method is mostly interesting when we choose one dimensional subspaces, we proposed to use a PCA and to project on the first directions for data embedded in euclidean spaces. For more complicated data such as graphs, we projected onto the Fiedler vector and obtained good results in an efficient way on a 3D mesh registration problem. More generally, \citeauthor{muzellec2019subspace} proposed to perform a gradient descent on the loss with respect to orthonormal matrices. This approach is non-convex and only guaranteed to converge to a local minimum. Designing such an algorithm, which would minimize alternatively between two transformations in the Stiefel manifold, is left for future works.

The subspace detour approach for transport problem is meaningful whenever one can identify subspaces that gather most of the information from the original distributions, while making the estimate more robust and with a better sample complexity as far as dimensions are lower. On the computational complexity side, and when we have only access to discrete data, the subspace detour approach brings better computational complexity solely when the subspaces are chosen as one dimensional. Indeed, otherwise, we have the same complexity for solving the subspace detour and solving directly the OT problem (since the complexity only depends on the number of samples). 
In this case, the 1D projection often gives distinct values for all the samples (for continuous valued data) and hence the Monge-Knothe coupling is exactly the coupling in 1D. As such, information is lost on the orthogonal spaces. It can be artificially recovered by quantizing the 1D values (as experimented in practice in \citet{muzellec2019subspace}), but the added value is not clear and deserves broader studies. If given absolutely continuous distributions \emph{wrt.} the Lebesgue measure however, this limit does not exist, but comes with the extra cost of being able to compute efficiently the projected measure onto the subspace, which might require discretization of the space and is therefore not practical in high dimensions.

We also proposed a new quadratic cost $\mathcal{HW}$ that we call Hadamard-Wasserstein, which allows to define a degenerated cost for which the optimal transport plan converges to a triangular coupling. However, this cost loses many properties compared to $W_2$ or $GW$, for which we are inclined to use these problems. Indeed, while $\mathcal{HW}$ is a quadratic cost, it uses an euclidean norm between the Hadamard product of vectors and requires the two spaces to be the same (in order to have the distance well defined). A work around in the case $X=\mathbb{R}^p$ and $Y=\mathbb{R}^q$ with $p\le q$ would be to ``lift'' the vectors in $\mathbb{R}^p$ into vectors in $\mathbb{R}^q$ with padding as it is proposed in \citet{titouan2019sliced}, or to project the vectors in $\mathbb{R}^q$ on $\mathbb{R}^p$ as in \citet{cai2020distances}. Yet for some applications where only the distance/similarity matrices are available, a different strategy still needs to be found.
Another concern is the limited invariance properties (only with respect to axial symmetry symmetry in our case). 
Nevertheless, we expect that such a cost can be of interest in cases where invariance to symmetry is a desired property, such as in \citep{nagar2019detecting}.

\bibliographystyle{plainnat}
\bibliography{references}


\appendix

\section{Subspace detours} \label{closed_forms}

\subsection{Proofs} \label{proof_props}

\begin{proof}[Proof of Proposition \ref{prop1}] \label{proof_prop1}
    $\forall \gamma\in\Pi_{E,F}(\mu,\nu)$, 
    \begin{equation*}
        \begin{aligned}
             &\iint L(x,x',y,y') \mathrm{d}\gamma(x,y)\mathrm{d}\gamma(x',y') \\ 
             &= \iint \Big(\iint L(x,x',y,y') \gamma_{E^\bot\times F^\bot|E\times F}((x_E,y_F),(\mathrm{d}x_{E^\bot},\mathrm{d}y_{F^\bot})) \gamma_{E^\bot\times F^\bot|E\times F}((x_E',y_F'),(\mathrm{d}x_{E^\bot}',\mathrm{d}y_{F^\bot}'))\\ &\Big)\mathrm{d}\gamma_{E\times F}^*(x_E,y_F) \mathrm{d}\gamma_{E\times F}^*(x_E',y_F')
        \end{aligned}
    \end{equation*}
    However, for $\gamma_{E\times F}^*$ a.e. $(x_E,y_F),(x_E',y_F')$,
    \begin{equation*}
        \begin{aligned}
            &\iint L(x,x',y,y') \gamma_{E^\bot\times F^\bot|E\times F}((x_E,y_F),(\mathrm{d}x_{E^\bot},\mathrm{d}y_{F^\bot})) \gamma_{E^\bot\times F^\bot|E\times F}((x_E',y_F'),(\mathrm{d}x_{E^\bot}',\mathrm{d}y_{F^\bot}')) \\
            &\geq \iint L(x,x',y,y') \gamma_{E^\bot\times F^\bot|E\times F}^*((x_E,y_F),(\mathrm{d}x_{E^\bot},\mathrm{d}y_{F^\bot})) \gamma_{E^\bot\times F^\bot|E\times F}^*((x_E',y_F'),(\mathrm{d}x_{E^\bot}',\mathrm{d}y_{F^\bot}'))
        \end{aligned}
    \end{equation*}
    by definition of the Monge-Knothe coupling. This is well optimal for subspace optimal plans.
\end{proof}

\begin{proof}[Proof of Proposition \ref{prop2}] \label{proof_prop2}
    Let $f:\mathbb{R}^p\to\mathbb{R}^p$ be an invariance of $GW$ on $E^\bot$, \emph{i.e.} $\forall x\in\mathbb{R}^p$, $f(x)=(x_E,f_{E^\bot}(x_{E^\bot}))$. We first deal with $L(x,x',y,y')=\big(\|x-x'\|_2^2 - \|y-y'\|_2^2\big)^2$, and therefore $f_{E^\bot}$ is either an isometry or a translation.
    
    From lemma 6 of \citet{paty2019subspace}, we know that $\Pi(f_\#\mu,\nu)=\{(f,Id)_\#\gamma|\ \gamma\in\Pi(\mu,\nu)\}$. We can rewrite
    \begin{equation*}
        \begin{aligned}
            \Pi_{E,F}(f_\#\mu,\nu)&=\{\gamma\in\Pi(f_\#\mu,\nu)|(\pi^E,\pi^F)_\#\gamma=\gamma_{E,F}^*\} \\
            &= \{(f,Id)_\#\gamma | \gamma\in\Pi(\mu,\nu), (\pi^E,\pi^F)_\#(f,Id)_\#\gamma=\gamma_{E,F}^*\} \\
            &=  \{(f,Id)_\#\gamma | \gamma\in\Pi(\mu,\nu), (\pi^E,\pi^F)_\#\gamma=\gamma_{E,F}^*\} \\
            &= \{(f,Id)_\#\gamma|\gamma\in\Pi_{E,F}(\mu,\nu)\}
        \end{aligned}
    \end{equation*}
    using $f=(Id_E, f_{E^\bot})$, $\pi^E\circ f = Id_E$ and $(\pi^E,\pi^F)_\#(f,Id)_\#\gamma=(\pi^E,\pi^F)_\#\gamma$.
    
    Now, for all $\gamma\in\Pi_{E,F}(f_\#\mu,\nu)$, there exists $\Tilde{\gamma}\in\Pi_{E,F}(\mu,\nu)$ such that $\gamma=(f,Id)_\#\Tilde{\gamma}$ and we can disintegrate $\Tilde{\gamma}$ with respect to $\gamma_{E,F}^*$
    \begin{equation*}
        \Tilde{\gamma} = \gamma_{E,F}^* \otimes K
    \end{equation*}
    with $K$ a probability kernel on $(E\times F, \mathcal{B}(E)\otimes \mathcal{B}(F))$.
    
    For $\gamma_{E,F}^*$ almost every $(x_E,y_F),\ (x_E',y_F')$, we have
    \begin{equation*}
        \begin{aligned}
            &\iint \big(\|x_E-x_E'\|_2^2+\|x_{E^\bot}-x_{E^\bot}'\|_2^2-\|y_F-y_F'\|_2^2-\|y_{F^\bot}-y_{F^\bot}'\|_2^2\big)^2 \\
            &\ (f_{E^\bot},Id)_\# K((x_E,y_F),(\mathrm{d}x_{E^\bot},\mathrm{d}y_{F^{\bot}})) (f_{E^\bot},Id)_\# K((x_E',y_F'),(\mathrm{d}x_{E^\bot}',\mathrm{d}y_{F^{\bot}}')) \\
            &= \iint \big(\|x_E-x_E'\|_2^2+\|f_{E^\bot}(x_{E^\bot})-f_{E^\bot}(x_{E^\bot}')\|_2^2-\|y_F-y_F'\|_2^2-\|y_{F^\bot}-y_{F^\bot}'\|_2^2\big)^2 \\
            &\ K((x_E,y_F),(\mathrm{d}x_{E^\bot},\mathrm{d}y_{F^{\bot}})) K((x_E',y_F'),(\mathrm{d}x_{E^\bot}',\mathrm{d}y_{F^{\bot}}')) \\
            &= \iint \big(\|x_E-x_E'\|_2^2+\|x_{E^\bot}-x_{E^\bot}'\|_2^2-\|y_F-y_F'\|_2^2-\|y_{F^\bot}-y_{F^\bot}'\|_2^2\big)^2 \\
            &\ K((x_E,y_F),(\mathrm{d}x_{E^\bot},\mathrm{d}y_{F^{\bot}})) K((x_E',y_F'),(\mathrm{d}x_{E^\bot}',\mathrm{d}y_{F^{\bot}}'))
        \end{aligned}
    \end{equation*}
    using in the last line that $\|f_{E^\bot}(x_{E^\bot})-f_{E^\bot}(x_{E^\bot}')\|_2 = \|x_{E^\bot}-x_{E^\bot}'\|_2$ since $d(x,y)=\|x-y\|_2$ is translation and rotation invariant ($d(Ox,Oy)=d(x,y)$ and $d(Tx,Ty)=d(x,y)$).
    
    By integrating with respect to $\gamma_{E,F}^*$, we obtain
    \begin{equation} \label{eq:equality_K}
        \begin{aligned}
        &\iint\Big(\iint \big(\|x-x'\|_2^2 - \|y-y'\|_2^2\big)^2 \\
        &\ (f_{E^\bot},Id)_\# K((x_E,y_F),(\mathrm{d}x_{E^\bot},\mathrm{d}y_{F^{\bot}})) (f_{E^\bot},Id)_\# K((x_E',y_F'),(\mathrm{d}x_{E^\bot}',\mathrm{d}y_{F^{\bot}}')) \Big) \mathrm{d}\gamma_{E,F}^*(x_E,y_F) \mathrm{d}\gamma_{E,F}^*(x_E',y_F') \\
        &= \iint \big(\|x-x'\|_2^2 - \|y-y'\|_2^2\big)^2\ \mathrm{d}\Tilde{\gamma}(x,y)\mathrm{d}\Tilde{\gamma}(x',y').
        \end{aligned}
    \end{equation}
    
    Now, we show that $\gamma=(f,Id)_\#\Tilde{\gamma} = \gamma_{E,F}^*\otimes (f_{E^\bot},Id)_\# K$. Let $\phi$ some bounded measurable function on $\mathbb{R}^p\times\mathbb{R}^q$,
    \begin{equation*}
        \begin{aligned}
            \int \phi(x,y)\mathrm{d}\gamma(x,y) &= \int \phi(x,y) \mathrm{d}((f,Id)_\#\Tilde{\gamma}(x,y)) \\
            &= \int \phi(f(x),y) \mathrm{d}\Tilde{\gamma}(x,y) \\
            &= \iint \phi(f(x),y) K\big((x_E,y_F),(\mathrm{d}x_{E^\bot}, \mathrm{d}y_{F^\bot})\big)\ \mathrm{d}\gamma_{E,F}^*(x_E,y_F) \\
            &= \iint \phi((x_E,f_{E^\bot}(x_{E^\bot})),y) K\big((x_E,y_F),(\mathrm{d}x_{E^\bot}, \mathrm{d}y_{F^\bot})\big)\ \mathrm{d}\gamma_{E,F}^*(x_E,y_F) \\
            &= \iint \phi(x,y) (f_{E^\bot},Id)_\# K\big((x_E,y_F),(\mathrm{d}x_{E^\bot}, \mathrm{d}y_{F^\bot})\big)\ \mathrm{d}\gamma_{E,F}^*(x_E,y_F).
        \end{aligned}
    \end{equation*}
    
    Hence, we can rewrite \eqref{eq:equality_K} as
    \begin{equation*}
        \iint \big(\|x-x'\|_2^2-\|y-y'\|_2^2\big)^2\ \mathrm{d}(f,Id)_\#\Tilde{\gamma}(x,y) \mathrm{d}(f,Id)_\#\Tilde{\gamma}(x',y') = \iint \big(\|x-x'\|_2^2-\|y-y'\|_2^2\big)^2\ \mathrm{d}\Tilde{\gamma}(x,y)\mathrm{d}\Tilde{\gamma}(x',y').
    \end{equation*}
    
    Now, by taking the infimum with respect to $\Tilde{\gamma}\in\Pi_{E,F}(\mu,\nu)$, we find
    \begin{equation*}
        GW_{E,F}^2(f_\#\mu,\nu) = GW_{E,F}^2(\mu,\nu).
    \end{equation*}
    
    For the inner product case, we can do the same proof for isometries.
\end{proof}

\subsection{Closed-form between Gaussians}

Let $q\le p$, $\mu=\mathcal{N}(m_\mu,\Sigma)\in\mathcal{P}(\mathbb{R}^p)$, $\nu=\mathcal{N}(m_\nu,\Lambda)\in\mathcal{P}(\mathbb{R}^q)$ two Gaussian measures with $\Sigma=P_\mu D_\mu P_\mu^T$ and $\Lambda=P_\nu D_\nu P_\nu^T$.

Let $E\subset \mathbb{R}^p$ be a subspace of dimension $k$ and $F\subset\mathbb{R}^q$ a subspace of dimension $k'$.

We represent $\Sigma$ in an orthonormal basis of $E\oplus E^\bot$, and $\Lambda$ in an orthonormal basis of $F\oplus F^\bot$, \emph{i.e.} $\Sigma = \begin{pmatrix}\Sigma_E & \Sigma_{EE^\bot} \\ \Sigma_{E^\bot E} & \Sigma_{E^\bot}\end{pmatrix}$. We denote $\Sigma/\Sigma_E = \Sigma_{E^\bot}-\Sigma_{EE^\bot}^T\Sigma_E^{-1}\Sigma_{EE^\bot}$ the Schur complement of $\Sigma$ with respect to $\Sigma_E$. We know that the conditionals of Gaussians are Gaussians, and of covariance the Schur complement (see \emph{e.g.} \citet{rasmussen2003gaussian} or \citet{von2014mathematical}). 


\subsubsection{Quadratic GW problem} \label{quadraticGW_Gaussians}

For $GW$ with $c(x,x')=\|x-x'\|_2^2$, we have for now no guarantee that there exists an optimal coupling which is a transport map. \citeauthor{salmona2021gromov} proposed to restrict the problem to the set of Gaussian couplings $\pi(\mu,\nu)\cap \mathcal{N}_{p+q}$ where $\mathcal{N}_{p+q}$ denotes the set of Gaussians in $\mathbb{R}^{p+q}$. In that case, the problem becomes
\begin{equation} \label{ggw}
    GGW^2_2(\mu,\nu)=\inf_{\gamma\in\Pi(\mu,\nu)\cap \mathcal{N}_{p+q}}\ \iint \big(\|x-x'\|_2^2 - \|y-y'\|_2^2\big)^2 \mathrm{d}\gamma(x,y)\mathrm{d}\gamma(x',y').
\end{equation}

In that case, they showed that an optimal solution is of the form $T(x)=m_\nu+P_\nu A P_\mu^T(x-m_\mu)$ with $A=\begin{pmatrix} \Tilde{I}_q D_\nu^{\frac12}(D_\mu^{(q)})^{-\frac12} & 0_{q,p-q} \end{pmatrix}$ and $\Tilde{I}_q$ of the form $\mathrm{diag}\big((\pm 1)_{i\le q}\big)$.

Since the problem is translation invariant, we can always solve the problem between the centered measures.

In the following, we suppose that $k=k'$. Let's denote $T_{E,F}$ the optimal transport map for \eqref{ggw} between $\mathcal{N}(0,\Sigma_E)$ and $\mathcal{N}(0,\Lambda_F)$. According to Theorem 4.1 in \citet{salmona2021gromov}, such a solution exists and is of the form \eqref{monge_map_gw}. We also denote $T_{E^\bot,F^\bot}$ the optimal transport map between $\mathcal{N}(0,\Sigma/\Sigma_E)$ and $\mathcal{N}(0,\Lambda/\Lambda_F)$ (which is well defined since we assumed $p\ge q$ and hence $p-k\ge q-k'$ since $k=k'$).

We know that the Monge-Knothe transport map will be a linear map $T_{\mathrm{MK}}(x)=Bx$ with $B$ a block triangular matrix of the form
\begin{equation*}
    B = \begin{pmatrix}
        T_{E,F} & 0_{k',p-k} \\
        C & T_{E^\bot,F^\bot}
    \end{pmatrix} \in \mathbb{R}^{q\times p},
\end{equation*}
with $C\in\mathbb{R}^{(q-k')\times k}$, and such that $B\Sigma B^T = \Lambda$ (to have well a transport map between $\mu$ and $\nu$).

Actually, 
\begin{equation*}
    B\Sigma B^T = \begin{pmatrix}
        T_{E,F}\Sigma_E T_{E,F}^T & T_{E,F}\Sigma_E C^T + T_{E,F}\Sigma_{EE^\bot} T_{E^\bot,F^\bot}^T \\
        (C\Sigma_E + T_{E^\bot,F^\bot}\Sigma_{E^\bot E})T_{E,F}^T & (C\Sigma_E+T_{E^\bot,F^\bot}\Sigma_{E^\bot E})C^T + (C\Sigma_{EE^\bot}+T_{E^\bot,F^\bot}\Sigma_{E^\bot})T_{E^\bot,F^\bot}^T.
    \end{pmatrix}
\end{equation*}

First, we have well $T_{E,F}\Sigma_E T_{E,F}^T = \Lambda_F$ as $T_{E,F}$ is a transport map between $\mu_E$ and $\nu_F$. Then,
\begin{equation*}
    B\Sigma B^T = \Lambda \iff \begin{cases}
        T_{E,F}\Sigma_E T_{E,F}^T = \Lambda_F \\
        T_{E,F}\Sigma_E C^T + T_{E,F}\Sigma_{EE^\bot} T_{E^\bot,F^\bot}^T = \Lambda_{FF^\bot} \\
        (C\Sigma_E + T_{E^\bot,F^\bot}\Sigma_{E^\bot E})T_{E,F}^T = \Lambda_{F^\bot F} \\
        (C\Sigma_E+T_{E^\bot,F^\bot}\Sigma_{E^\bot E})C^T + (C\Sigma_{EE^\bot}+T_{E^\bot,F^\bot}\Sigma_{E^\bot})T_{E^\bot,F^\bot}^T = \Lambda_{F^\bot}.
    \end{cases}
\end{equation*}

We have 
\begin{equation*}
    (C\Sigma_E + T_{E^\bot,F^\bot}\Sigma_{E^\bot E})T_{E,F}^T = \Lambda_{F^\bot F} \iff C\Sigma_E T_{E,F}^T = \Lambda_{F^\bot F} - T_{E^\bot, F^\bot}\Sigma_{E^\bot E} T_{E,F}^T.
\end{equation*}

As $k=k'$, $\Sigma_E T_{E,F}^T\in\mathbb{R}^{k\times k}$ and is invertible (as $\Sigma_E$ and $\Lambda_F$ are positive definite and $T_{E,F}=P_{\mu_E}A_{E,F}P_{\nu_F}$ with $A_{E,F}=\begin{pmatrix} \Tilde{I}_k D_{\nu_F}^{\frac11} D_{\mu_E}^{-\frac12} \end{pmatrix}$ with positive values on the diagonals. Hence, we have 
\begin{equation*}
    C = (\Lambda_{F^\bot F} (T_{E,F}^T)^{-1} - T_{E^\bot,F^\bot}\Sigma_{E^\bot E})\Sigma_E^{-1}.
\end{equation*}

Now, we still have to check the last two equations. First, 
\begin{equation*}
    \begin{aligned}
        T_{E,F}\Sigma_E C^T + T_{E,F}\Sigma_{EE^\bot} T_{E^\bot,F^\bot}^T &= T_{E,F}\Sigma_E \Sigma_E^{-1} T_{E,F}^{-1}\Lambda_{F^\bot F}^T  - T_{E,F}\Sigma_E \Sigma_E^{-1} \Sigma_{E^\bot E}^T T_{E^\bot, F^\bot}^T+ T_{E,F}\Sigma_{EE^\bot} T_{E^\bot,F^\bot}^T \\
        &= \Lambda_{FF^\bot}.
    \end{aligned}
\end{equation*}

And for the last equation,

\begin{equation*}
    \begin{aligned}
        &(C\Sigma_E+T_{E^\bot,F^\bot}\Sigma_{E^\bot E})C^T + (C\Sigma_{EE^\bot}+T_{E^\bot,F^\bot}\Sigma_{E^\bot})T_{E^\bot,F^\bot}^T \\
        &= (\Lambda_{F^\bot F} (T_{E,F}^T)^{-1} - T_{E^\bot,F^\bot}\Sigma_{E^\bot E}+T_{E^\bot,F^\bot}\Sigma_{E^\bot E}) \Sigma_E^{-1} (T_{E,F}^{-1}\Lambda_{F^\bot F}^T-\Sigma_{E^\bot E}^T T_{E^\bot,F^\bot}^T) \\
        &\ +\Lambda_{F^\bot F} (T_{E,F}^T)^{-1} \Sigma_E^{-1}\Sigma_{EE^\bot} T_{E^\bot, F^\bot}^T - T_{E^\bot,F^\bot}\Sigma_{E^\bot E}\Sigma_E^{-1} \Sigma_{EE^\bot} T_{E^\bot, F^\bot}^T + T_{E^\bot, F^\bot}\Sigma_{E^\bot} T_{E^\bot, F^\bot}^T \\
        &= \Lambda_{F^\bot F} (T_{E,F}^T )^{-1} \Sigma_E^{-1} T_{E,F}^{-1} \Lambda_{F^\bot F}^T - \Lambda_{F^\bot F} (T_{E,F}^T)^{-1} \Sigma_E^{-1}\Sigma_{E^\bot E}^T T_{E^\bot, F^\bot}^T - T_{E^\bot,F^\bot}\Sigma_{E^\bot E}\Sigma_E^{-1}T_{E,F}^{-1}\Lambda_{F^\bot F}^T \\
        &\ + T_{E^\bot,F^\bot}\Sigma_{E^\bot E} \Sigma_E^{-1} \Sigma_{E^\bot E}^T T_{E^\bot,F^\bot}^T + T_{E^\bot,F^\bot}\Sigma_{E^\bot E} \Sigma_E^{-1} T_{E,F}^{-1} \Lambda_{F^\bot F}^T - T_{E^\bot,F^\bot} \Sigma_{E^\bot E} \Sigma_E^{-1} \Sigma_{E^\bot E}^T T_{E^\bot F^\bot}^T \\
        &\ + \Lambda_{F^\bot F} (T_{E,F}^T)^{-1}\Sigma_E^{-1}\Sigma_{EE^\bot} T_{E^\bot,F^\bot}^T - T_{E^\bot,F^\bot}\Sigma_{E^\bot E}\Sigma_E^{-1}\Sigma_{E^\bot E}^T T_{E^\bot,F^\bot}^T + T_{E^\bot,F^\bot}\Sigma_{ E^\bot}T_{E^\bot,F^\bot}^T \\
        &= \Lambda_{F^\bot F} (T_{E,F}^T)^{-1}\Sigma_E^{-1}T_{E,F}^{-1}\Lambda_{F^\bot F}^T - T_{E^\bot,F^\bot} \Sigma_{E^\bot E}\Sigma_E^{-1} \Sigma_{E^\bot E}^T T_{E^\bot,F^\bot}^T + T_{E^\bot,F^\bot}\Sigma_{E^\bot} T_{E^\bot,F^\bot}^T
    \end{aligned}
\end{equation*}
Now, using that $(T_{E,F}^T)^{-1}\Sigma_E^{-1}T_{E,F}^{-1} = (T_{E,F}\Sigma_E T_{E,F}^T)^{-1} = \Lambda_F^{-1}$ and $\Sigma_{E^\bot}-\Sigma_{E^\bot E}\Sigma_E^{-1}\Sigma_{E^\bot E}^T = \Sigma/\Sigma_E$, we have
\begin{equation*}
    \begin{aligned}
        &(C\Sigma_E+T_{E^\bot,F^\bot}\Sigma_{E^\bot E})C^T + (C\Sigma_{EE^\bot}+T_{E^\bot,F^\bot}\Sigma_{E^\bot})T_{E^\bot,F^\bot}^T \\
        &= \Lambda_{F^\bot F}\Lambda_F^{-1}\Lambda_{F^\bot F}^T + T_{E^\bot,F^\bot} (\Sigma_{E^\bot}-\Sigma_{E^\bot E}\Sigma_E^{-1}\Sigma_{E^\bot E}^T) T_{E^\bot, F^\bot}^T \\
        &= \Lambda_{F^\bot F}\Lambda_F^{-1}\Lambda_{F^\bot F}^T + \Lambda/\Lambda_F \\
        &= \Lambda_{F^\bot}
    \end{aligned}
\end{equation*}

Then $\pi_{\mathrm{MK}}$ is of the form $(Id,T_{\mathrm{MK}})_\#\mu$ with
\begin{equation*}
    T_{\mathrm{MK}}(x) = m_\nu+B(x-m_\mu).
\end{equation*}

\subsubsection{Closed-form between Gaussians for Monge-Independent} \label{mi_gaussians}

Suppose $k\ge k'$ in order to be able to define the OT map between $\mu_E$ and $\nu_F$.

For the Monge-Independent plan $\pi_{\mathrm{MI}}=\gamma_{E,F}^*\otimes(\mu_{E^\bot|E}\otimes \nu_{F^\bot|F})$, let $(X,Y)\sim \pi_{\mathrm{MI}}$. We know that $\pi_{\mathrm{MI}}$ is a degenerate Gaussian with a covariance of the form
\begin{equation*}
    \mathrm{Cov}(X,Y) = \begin{pmatrix}
        \mathrm{Cov}(X) & C \\
        C^T & \mathrm{Cov}(Y)
    \end{pmatrix}
\end{equation*}
where $\mathrm{Cov}(X)=\Sigma$ and $\mathrm{Cov}(Y)=\Lambda$. Moreover, we know that $C$ is of the form
\begin{equation*}
    \begin{pmatrix}
        \mathrm{Cov}(X_E,Y_F) & \mathrm{Cov}(X_E,Y_{F^\bot}) \\
        \mathrm{Cov}(X_{E^\bot},Y_F) & \mathrm{Cov}(X_{E^\bot},Y_{F^\bot})
    \end{pmatrix}.
\end{equation*}
Let's assume that $m_\mu=m_\nu=0$, then 
\begin{equation*}
    \begin{aligned}
        \mathrm{Cov}(X_E,Y_F) &= \mathrm{Cov}(X_E,T_{E,F}X_E) = \mathbb{E}[X_EX_E^T]T_{E,F}^T = \Sigma_E T_{E,F}^T,
    \end{aligned}
\end{equation*}
\begin{equation*}
    \begin{aligned}
        \mathrm{Cov}(X_E,Y_{F^\bot}) &= \mathbb{E}[X_E Y_{F^\bot}^T] \\
        &= \mathbb{E}[\mathbb{E}[X_E Y_{F^\bot}^T|X_E,Y_F]] \\
        &= \mathbb{E}[X_E\mathbb{E}[Y_{F^\bot}^T|Y_F]]
    \end{aligned}
\end{equation*}
since $Y_F=T_{E,F}X_E$, $X_E$ is $\sigma(Y_F)$-measurable. Now, using the equation (A.6) from \citet{rasmussen2003gaussian}, we have
\begin{equation*}
    \begin{aligned}
        \mathbb{E}[Y_{F^\bot}|Y_F] &= \Lambda_{F^\bot F}\Lambda_F^{-1}Y_F \\
        &= \Lambda_{F^\bot F}\Lambda_F^{-1}T_{E,F} X_E
    \end{aligned}
\end{equation*}
and 
\begin{equation*}
    \mathbb{E}[X_{E^\bot}|X_E] = \Sigma_{E^\bot E} \Sigma_E^{-1} X_E.
\end{equation*}

Hence,
\begin{equation*}
    \begin{aligned}
        \mathrm{Cov}(X_E,Y_{F^\bot}) &=\mathbb{E}[X_E\mathbb{E}[Y_{F^\bot}^T|Y_F]] \\
        &= \mathbb{E}[X_E X_E^T] T_{E,F}^T \Lambda_F^{-1} \Lambda_{F^\bot F}^T \\
        &= \Sigma_E T_{E,F}^T \Lambda_F^{-1} \Lambda_{F^\bot F}^T.
    \end{aligned}
\end{equation*}

We also have
\begin{equation*}
    \mathrm{Cov}(X_{E^\bot},Y_F) = \mathbb{E}[X_{E^\bot} X_E^T T_{E,F}^T] = \Sigma_{E^\bot E}T_{E,F}^T,
\end{equation*}
and
\begin{equation*}
    \begin{aligned}
        \mathrm{Cov}(X_{E^\bot}, Y_{F^\bot}) &= \mathbb{E}[X_{E^\bot} Y_{F^\bot}^T] \\
        &= \mathbb{E}[\mathbb{E}[X_{E^\bot}Y_{F^\bot}^T|X_E,Y_F]] \\
        &= \mathbb{E}[\mathbb{E}[X_{E^\bot}|X_E]\mathbb{E}[Y_{F^\bot}^T|Y_F]]\ \text{ by independence} \\
        &= \mathbb{E}[\Sigma_{E^\bot E}\Sigma_E^{-1} X_E X_E^T T_{E,F}^T \Lambda_F^{-1}\Lambda_{F^\bot F}^T] \\
        &= \Sigma_{E^\bot E}T_{E,F}^T \Lambda_F^{-1}\Lambda_{F^\bot F}^T.
    \end{aligned}
\end{equation*}

Finally, we find 
\begin{equation*}
    C = \begin{pmatrix}
        \Sigma_E T_{E,F}^T & \Sigma_E T_{E,F}^T \Lambda_F^{-1} \Lambda_{F^\bot F}^T \\
        \Sigma_{E^\bot E}T_{E,F}^T & \Sigma_{E^\bot E}T_{E,F}^T \Lambda_F^{-1}\Lambda_{F^\bot F}^T
    \end{pmatrix}.
\end{equation*}

By taking orthogonal bases $(V_E,V_{E^\bot})$ and $(V_F,V_{F^\bot})$, we can put it in a more compact way such as in Proposition 4 in \citet{muzellec2019subspace}:
\begin{equation*}
    C = (V_E\Sigma_E+V_{E^\bot}\Sigma_{E^\bot E})T_{E,F}^T(V_F^T+\Lambda_F^{-1}\Lambda_{F^\bot F}^T V_{F^\bot}^T).
\end{equation*}

To check it, just expand the terms and see that $C_{E,F} = V_E C V_F^T$.

\section{Knothe-Rosenblatt} 

\subsection{Properties of \eqref{HadamardWasserstein}}

\begin{proof}[Proof of Proposition \ref{prop_hadamard}] \label{proof_Hadamard}
    Let $\mu,\nu\in\mathcal{P}(\mathbb{R}^d)$,
    \begin{enumerate}
        \item $(x,x')\mapsto x\odot x'$ is a continuous map, therefore $L$ is lower semi-continuous. Hence, by applying lemma 2.2.1 of \citep{vayer2020contribution}, we have that $\gamma\mapsto \iint L(x,x',y,y')\mathrm{d}\gamma(x,y)\mathrm{d}\gamma(x',y')$ is lower semi-continuous for the weak convergence of measures.
        
        Now, as $\Pi(\mu,\nu)$ is a compact set (see the proof of Theorem 1.7 in \cite{santambrogio2015optimal} for the Polish space case, and of Theorem 1.4 for the compact metric space), and $\gamma\mapsto \iint L\mathrm{d}\gamma\mathrm{d}\gamma$ is lower semi-continuous for the weak convergence, we can apply the Weierstrass theorem (Memo 2.2.1 in \cite{vayer2020contribution}) which states that \eqref{HadamardWasserstein} always admits a minimizer.
        \item See Theorem 16 in \cite{chowdhury2019gromov}.
        \item For invariances, we first look at the properties that must be satisfied by $T$ in order to have: $\forall x,x',\ f(x,x')=f(T(x),T(x'))$ where $f:(x,x')\mapsto x\odot x'$.
        
        We find that $\forall x\in\mathbb{R}^d,\ \forall 1\le i\le d,\ |[T(x)]_i|=|x_i|$ because, denoting $(e_i)_{i=1}^d$ as the canonical basis, we have
        \begin{equation*}
            f(e_i,x)=x e_i = f(T(e_i),T(x)) \Longrightarrow [T(e_i)]_i[T(x)]_i = x_i \Longrightarrow |[T(x)]_i|=|x_i|
        \end{equation*}
         as $f(e_i,e_i)=f(T(e_i),T(e_i)) \Longrightarrow [T(e_i)]_i^2 = 1$.
         
        If we take for $T$ the reflection with respect to axis, then it satisfies well $f(x,x')=f(T(x),T(x'))$.
        Moreover, it is well an equivalent relation, and therefore we have a distance on the quotient space.
    \end{enumerate}
\end{proof}

\begin{proposition}
    In a slightly more general setting, let $X_0=X_1=\mathbb{R}^d$, functions $f_0,\ f_1$ from $\mathbb{R}^d\times \mathbb{R}^d$ to $\mathbb{R}^d$ and measures $\mu_0\in\mathcal{P}(X_0)$, $\mu_1\in\mathcal{P}(X_1)$. Then the family $\mathcal{X}_t=(X_0\times X_1, f_t, \gamma^*)$ defines a geodesic between $\mathcal{X}_0$ and $\mathcal{X}_1$, where $\gamma^*$ is the optimal coupling of $\mathcal{HW}$ between $\mu_0$ and $\mu_1$, and 
    \begin{equation*}
        f_t((x_0,x_0'),(x_1,x_1'))=(1-t)f_0(x_0,x_0')+tf_1(x_1,x_1').
    \end{equation*}
\end{proposition}

\begin{proof}
    See Theorem 3.1 in \cite{sturm2012space}.
\end{proof}

\subsection{Proof of Theorem \ref{ThGWKR}} \label{ProofThGWKR}

We first recall a useful theorem.

\begin{theorem}[Theorem 2.8 in \citet{billingsley2013convergence}] \label{billingsley}
    Let $\Omega=X\times Y$ be a separable space, and let $P,P_n\in\mathcal{P}(\Omega)$ with marginals $P_X$ (respectively $P_{n,X}$) and $P_Y$ (respectively $P_{n,Y}$). Then, $P_{n,X}\otimes P_{n,Y} \xrightarrow[]{\mathcal{D}}P$ if and only if $P_{n,X}\xrightarrow[]{\mathcal{D}}P_X$, $P_{n,Y}\xrightarrow[]{\mathcal{D}}P_Y$ and $P=P_X\otimes P_Y$. 
\end{theorem}

\begin{proof}[Proof of Theorem~\ref{ThGWKR}] 
    The following proof is mainly inspired by the proof of Theorem ~\ref{th_KnotheToBrenier} in \citep{carlier2010knothe}[Theorem 2.1], \citep{bonnotte2013unidimensional}[Theorem 3.1.6] and \citep{santambrogio2015optimal}[Theorem 2.23].
    
    Let $\mu,\nu\in\mathcal{P}(\mathbb{R}^d)$, absolutely continuous, with finite fourth moments and compact supports. We recall the problem $\mathcal{HW}_t$,
    \begin{equation*}
        \mathcal{HW}_t^2(\mu,\nu) = \inf_{\gamma\in\Pi(\mu,\nu)}\ \iint \Big( \prod_{i=1}^{k-1} \lambda_t^{(i)}\Big)\ (x_kx_k'-y_ky_k')^2\ \mathrm{d}\gamma_t(x,y)\mathrm{d}\gamma_t(x',y'),
    \end{equation*}
    with $\forall t>0,\ \forall i\in\{1,\dots,d-1\}$, $\lambda_t^{(i)}>0$ and $\lambda_t^{(i)}\xrightarrow[t\to 0]{}0$.
    
    First, let's denote $\gamma_t$ the optimal coupling for $\mathcal{HW}_t$ for all $t>0$. We want to show that $\gamma_t\xrightarrow[t\rightarrow 0]{\mathcal{D}}\gamma_K$ with $\gamma_K=(Id\times T_K)_\#\mu$ and $T_K$ our alternate Knothe-Rosenblatt rearrangement. Let $\gamma\in\Pi(\mu,\nu)$ such that $\gamma_t\xrightarrow[t\rightarrow 0]{\mathcal{D}}\gamma$ (true up to subsequence as $\{\mu\}$ and $\{\nu\}$ are tight in $\mathcal{P}(X)$ and $\mathcal{P}(Y)$ if $X$ and $Y$ are polish space, therefore, by \citep{villani2008optimal}[Lemma 4.4], $\Pi(\mu,\nu)$ is a tight set, and we can apply the Prokhorov theorem \citep{santambrogio2015optimal}[Box 1.4] on $(\gamma_t)_t$ and extract a subsequence)).
    
    \proofpart{1}{}
    
    First, let's notice that:
    \begin{align*}
        \mathcal{HW}_t^2(\mu,\nu) &= \iint \sum_{k=1}^d\Big( \prod_{i=1}^{k-1} \lambda_t^{(i)}\Big)\ (x_kx_k'-y_ky_k')^2\ \mathrm{d}\gamma_t(x,y)\mathrm{d}\gamma_t(x',y') \\
        &= \iint (x_1x_1'-y_1y_1')^2\ \mathrm{d}\gamma_t(x,y)\mathrm{d}\gamma_t(x',y') + \iint \sum_{k=2}^d\Big( \prod_{i=1}^{k-1} \lambda_t^{(i)}\Big)\ (x_kx_k'-y_ky_k')^2\ \mathrm{d}\gamma_t(x,y)\mathrm{d}\gamma_t(x',y'). \\
    \end{align*}
    Moreover, as $\gamma_t$ is the optimal coupling between $\mu$ and $\nu$, and $\gamma_K\in\Pi(\mu,\nu)$,
    \begin{align*}
        \mathcal{HW}_t^2(\mu,\nu) &\le \iint \sum_{k=1}^d\Big( \prod_{i=1}^{k-1} \lambda_t^{(i)}\Big)\ (x_kx_k'-y_ky_k')^2\ \mathrm{d}\gamma_K(x,y)\mathrm{d}\gamma_K(x',y') \\
        &= \iint (x_1x_1'-y_1y_1')^2\ \mathrm{d}\gamma_K(x,y)\mathrm{d}\gamma_K(x',y') + \iint \sum_{k=2}^d\Big( \prod_{i=1}^{k-1} \lambda_t^{(i)}\Big)\ (x_kx_k'-y_ky_k')^2\ \mathrm{d}\gamma_K(x,y)\mathrm{d}\gamma_K(x',y'). \\
    \end{align*}
    
    In our case, we have $\gamma_t\xrightarrow[t\rightarrow 0]{\mathcal{D}}\gamma$, thus, by Theorem \ref{billingsley}, we have $\gamma_t\otimes\gamma_t\xrightarrow[t\rightarrow0]{\mathcal{D}}\gamma\otimes\gamma$. Using the fact that $\forall i,\ \lambda_t^{(i)}\xrightarrow[t\rightarrow0]{}0$ (and lemma 1.8 of \citet{santambrogio2015optimal}, since we are on compact support, we can bound the cost (which is continuous) by its max), we obtain the following inequality
    \begin{equation*}
        \iint (x_1x_1'-y_1y_1')^2\ \mathrm{d}\gamma(x,y)\mathrm{d}\gamma(x',y') \le \iint (x_1x_1'-y_1y_1')^2\ \mathrm{d}\gamma_K(x,y)\mathrm{d}\gamma_K(x',y').
    \end{equation*}
    By denoting $\gamma^1$ and $\gamma_K^1$ the marginals on the first variables, we can use the projection $\pi^1(x,y)=(x_1,y_1)$, such as $\gamma^1=\pi^1_\#\gamma$ and $\gamma_K^1 = \pi^1_\#\gamma_K$. Hence, we get
    \begin{equation*}
        \iint (x_1x_1'-y_1y_1')^2\ \mathrm{d}\gamma^1(x_1,y_1)\mathrm{d}\gamma^1(x_1',y_1') \le \iint (x_1x_1'-y_1y_1')^2\ \mathrm{d}\gamma^1_K(x_1,y_1)\mathrm{d}\gamma^1_K(x_1',y_1').
    \end{equation*}
    However, $\gamma_K^1$ was constructed in order to be the unique optimal map for this cost (either $T_{asc}$ or $T_{desc}$ according to theorem \citep{vayer2020contribution}[Theorem 4.2.4]). Thus, we can deduce that $\gamma^1=(Id\times T_K^1)_\#\mu^1 = \gamma_K^1$.
    
    \proofpart{2}{}
    
    We know that for any $t>0$, $\gamma_t$ and $\gamma_K$ share the same marginals. Thus, as previously, $\pi^1_\#\gamma_t$ should have a cost worse than $\pi^1_\#\gamma_K$, which translates to
    \begin{align*}
        \iint (x_1x_1'-y_1y_1')^2\ \mathrm{d}\gamma^1_K(x_1,y_1)\mathrm{d}\gamma^1_K(x_1',y_1') &= \iint (x_1x_1'-y_1y_1')^2\ \mathrm{d}\gamma^1(x_1,y_1)\mathrm{d}\gamma^1(x_1',y_1')\\
        &\le \iint (x_1x_1'-y_1y_1')^2\ \mathrm{d}\gamma^1_t(x_1,y_1)\mathrm{d}\gamma^1_t(x_1',y_1').
    \end{align*}
    Therefore, we have the following inequality,
    \begin{align*}
        &\iint (x_1x_1'-y_1y_1')^2\ \mathrm{d}\gamma^1(x,y)\mathrm{d}\gamma^1(x',y') + \iint \sum_{k=2}^d\Big( \prod_{i=1}^{k-1} \lambda_t^{(i)}\Big)\ (x_kx_k'-y_ky_k')^2\ \mathrm{d}\gamma_t(x,y)\mathrm{d}\gamma_t(x',y') \\
        &\le \mathcal{HW}_t^2(\mu,\nu) \\
        &\le \iint (x_1x_1'-y_1y_1')^2\ \mathrm{d}\gamma^1(x,y)\mathrm{d}\gamma^1(x',y') + \iint \sum_{k=2}^d\Big( \prod_{i=1}^{k-1} \lambda_t^{(i)}\Big)\ (x_kx_k'-y_ky_k')^2\ \mathrm{d}\gamma_K(x,y)\mathrm{d}\gamma_K(x',y').
    \end{align*}
    We can substract the first term and factorize by $\lambda_t^{(1)}>0$,
    \begin{align*}
        &\iint \sum_{k=2}^d\Big( \prod_{i=1}^{k-1} \lambda_t^{(i)}\Big)\ (x_kx_k'-y_ky_k')^2\ \mathrm{d}\gamma_t(x,y)\mathrm{d}\gamma_t(x',y') \\
        &= \lambda_t^{(1)}\Big( \iint (x_2x_2'-y_2y_2')^2\ \mathrm{d}\gamma_t(x,y)\mathrm{d}\gamma_t(x',y')+\iint \sum_{k=3}^d\Big( \prod_{i=2}^{k-1} \lambda_t^{(i)}\Big)\ (x_kx_k'-y_ky_k')^2\ \mathrm{d}\gamma_t(x,y)\mathrm{d}\gamma_t(x',y') \Big) \\
        &\le \lambda_t^{(1)}\Big( \iint (x_2x_2'-y_2y_2')^2\ \mathrm{d}\gamma_K(x,y)\mathrm{d}\gamma_K(x',y')+\iint \sum_{k=3}^d\Big( \prod_{i=2}^{k-1} \lambda_t^{(i)}\Big)\ (x_kx_k'-y_ky_k')^2\ \mathrm{d}\gamma_K(x,y)\mathrm{d}\gamma_K(x',y') \Big).
    \end{align*}
    By dividing by $\lambda_t^{(1)}$ and by taking the limit $t\rightarrow 0$ as in the first part, we get
    \begin{equation} \label{inequalityBeforeDisintegration}
        \iint (x_2x_2'-y_2y_2')^2\ \mathrm{d}\gamma(x,y)\mathrm{d}\gamma(x',y') \le \iint (x_2x_2'-y_2y_2')^2\ \mathrm{d}\gamma_K(x,y)\mathrm{d}\gamma_K(x',y').
    \end{equation}
    Now, the 2 terms depend only on $(x_2,y_2)$ and $(x_2',y_2')$. We will project on the two first coordinates, \emph{i.e.} let $\pi^{1,2}(x,y)=((x_1,x_2),(y_1,y_2))$ and $\gamma^{1,2}=\pi^{1,2}_\#\gamma$, $\gamma^{1,2}_K=\pi^{1,2}_\#\gamma_K$. Using the disintegration of measures, we know that there exist kernels $\gamma^{2|1}$ and $\gamma^{2|1}_K$ such that $\gamma^{1,2}=\gamma^1 \otimes \gamma^{2|1}$ and $\gamma^{1,2}_K = \gamma_K^1 \otimes \gamma^{2|1}_K$, where
    \begin{equation*}
        \forall A\in\mathcal{B}(X\times Y),\ \mu\otimes K(A) = \iint \mathbb{1}_A(x,y) K(x,\mathrm{d}y)\mu(\mathrm{d}x).
    \end{equation*}
    We can rewrite the previous equation (\ref{inequalityBeforeDisintegration}) as
    \begin{equation} \label{inequality_1}
        \begin{aligned}
            &\iint (x_2x_2'-y_2y_2')^2\ \mathrm{d}\gamma(x,y)\mathrm{d}\gamma(x',y')  \\
            &= \iiiint (x_2x_2'-y_2y_2')^2\ \gamma^{2|1}((x_1,y_1),(\mathrm{d}x_2,\mathrm{d}y_2))\gamma^{2|1}((x'_1,y'_1),(\mathrm{d}x'_2,\mathrm{d}y'_2))\mathrm{d}\gamma^1(x_1,y_1)\mathrm{d}\gamma^1(x_1',y_1') \\
            &\le \iiiint (x_2x_2'-y_2y_2')^2\ \gamma_K^{2|1}((x_1,y_1),(\mathrm{d}x_2,\mathrm{d}y_2))\gamma_K^{2|1}((x'_1,y'_1),(\mathrm{d}x'_2,\mathrm{d}y'_2))\mathrm{d}\gamma_K^1(x_1,y_1)\mathrm{d}\gamma_K^1(x_1',y_1').
        \end{aligned}
    \end{equation}
    
    Now, we will assume at first that the marginals of $\gamma^{2|1}((x_1,y_1),\cdot)$ are well $\mu^{2|1}(x_1,\cdot)$ and $\nu^{2|1}(y_1,\cdot)$. Then, by definition of $\gamma_K^{2|1}$, as it is optimal for the $GW$ cost with inner products, we have for all $(x_1,y_1), (x_1',y_1')$,
    \begin{equation} \label{inequality_KR}
        \begin{aligned}
            &\iint (x_2x_2'-y_2y_2')^2\ \gamma_K^{2|1}((x_1,y_1),(\mathrm{d}x_2,\mathrm{d}y_2))\gamma_K^{2|1}((x'_1,y'_1),(\mathrm{d}x'_2,\mathrm{d}y'_2)) \\ 
            &\le \iint (x_2x_2'-y_2y_2')^2\ \gamma^{2|1}((x_1,y_1),(\mathrm{d}x_2,\mathrm{d}y_2))\gamma^{2|1}((x'_1,y'_1),(\mathrm{d}x'_2,\mathrm{d}y'_2)).
        \end{aligned}
    \end{equation}
    Moreover, we know from the first part that $\gamma^1=\gamma_K^1$, then by integrating with respect to $(x_1,y_1)$ and $(x_1',y_1')$, we have
    \begin{equation} \label{inequality_2}
        \begin{aligned}
            &\iiiint (x_2x_2'-y_2y_2')^2\ \gamma_K^{2|1}((x_1,y_1),(\mathrm{d}x_2,\mathrm{d}y_2))\gamma_K^{2|1}((x'_1,y'_1),(\mathrm{d}x'_2,\mathrm{d}y'_2))\mathrm{d}\gamma^1(x_1,y_1)\mathrm{d}\gamma^1(x_1',y_1')  \\ 
            &\le \iiiint (x_2x_2'-y_2y_2')^2\ \gamma^{2|1}((x_1,y_1),(\mathrm{d}x_2,\mathrm{d}y_2))\gamma^{2|1}((x'_1,y'_1),(\mathrm{d}x'_2,\mathrm{d}y'_2))\mathrm{d}\gamma^1(x_1,y_1)\mathrm{d}\gamma^1(x_1',y_1'). 
        \end{aligned}
    \end{equation}
    By \eqref{inequality_1} and \eqref{inequality_2}, we deduce that we have an equality and we get
    \begin{equation} \label{equality}
        \begin{aligned}
            &\iint\Big(\iint (x_2x_2'-y_2y_2')^2\ \gamma^{2|1}((x_1,y_1),(\mathrm{d}x_2,\mathrm{d}y_2))\gamma^{2|1}((x'_1,y'_1),(\mathrm{d}x'_2,\mathrm{d}y'_2)) \\
            &- \iint(x_2x_2'-y_2y_2')^2\ \gamma^{2|1}_K((x_1,y_1),(\mathrm{d}x_2,\mathrm{d}y_2))\gamma^{2|1}_K((x'_1,y'_1),(\mathrm{d}x'_2,\mathrm{d}y'_2)) \Big)\mathrm{d}\gamma^1(x_1,y_1)\mathrm{d}\gamma^1(x_1',y_1') = 0.
        \end{aligned}
    \end{equation}
    However, we know by \eqref{inequality_KR} that the middle part of \eqref{equality} is nonnegative, thus we have for $\gamma^1$-a.e. $(x_1,y_1),(x_1',y_1')$, 
    \begin{align*}
        &\iint(x_2x_2'-y_2y_2')^2\ \gamma^{2|1}_K((x_1,y_1),(\mathrm{d}x_2,\mathrm{d}y_2))\gamma^{2|1}_K((x'_1,y'_1),(\mathrm{d}x'_2,\mathrm{d}y'_2)) \\ 
        &= \iint(x_2x_2'-y_2y_2')^2\ \gamma^{2|1}((x_1,y_1),(\mathrm{d}x_2,\mathrm{d}y_2))\gamma^{2|1}((x'_1,y'_1),(\mathrm{d}x'_2,\mathrm{d}y'_2)).
    \end{align*}
    From that, we can conclude as in the first part that $\gamma^{2|1}=\gamma_K^{2|1}$ (by unicity of the optimal map). And thus $\gamma^{1,2}=\gamma_K^{1,2}$.
    
    Now, we still have to show that the marginals of $\gamma^{2|1}((x_1,y_1),\cdot)$ and $\gamma_K^{2,1}((x_1,y_1),\cdot)$ are well the same, \emph{i.e.} $\mu^{2|1}(x_1,\cdot)$ and $\nu^{2|1}(y_1,\cdot)$. Let $\phi$ and $\psi$ be continuous functions, then we have to show that for $\gamma^1$-a.e. $(x_1,y_1)$, we have
    \begin{equation*}
        \begin{cases}
            \int \phi(x_2)\gamma^{2|1}((x_1,y_1),(\mathrm{d}x_2,\mathrm{d}y_2)) = \int \phi(x_2)\mu^{2|1}(x_1,\mathrm{d}x_2) \\
            \int \psi(y_2)\gamma^{2|1}((x_1,y_1),(\mathrm{d}x_2,\mathrm{d}y_2)) = \int \psi(y_2)\nu^{2|1}(y_1,\mathrm{d}y_2).
        \end{cases}
    \end{equation*}
    As we want to prove it for $\gamma^1$-a.e. $(x_1,y_1)$, it is sufficient to prove that for all continuous function $\xi$,
    \begin{equation*}
        \begin{cases}
            \iint \xi(x_1,y_1)\phi(x_2)\gamma^{2|1}((x_1,y_1),(\mathrm{d}x_2,\mathrm{d}y_2)) \mathrm{d}\gamma^1(x_1,y_1) = \iint \xi(x_1,y_1) \phi(x_2)\mu^{2|1}(x_1,\mathrm{d}x_2)\mathrm{d}\gamma^1(x_1,y_1)  \\
            \iint \xi(x_1,y_1) \psi(y_2)\gamma^{2|1}((x_1,y_1),(\mathrm{d}x_2,\mathrm{d}y_2)) \mathrm{d}\gamma^1(x_1,y_1) = \iint \xi(x_1,y_1) \psi(y_2)\nu^{2|1}(y_1,\mathrm{d}y_2) \mathrm{d}\gamma^1(x_1,y_1).
        \end{cases}
    \end{equation*}
    First, we can use the projections $\pi_x(x,y)=x$ and $\pi_y(x,y)=y$. Moreover, we know that $\gamma^1 = (Id\times T_K^1)_\#\mu^1$. The alternate Knothe-Rosenblatt rearrangement is, as the usual one, bijective (because $\mu$ and $\nu$ are absolutely continuous), and thus, as we suppose that $\nu$ satisfies the same hypothesis than $\mu$, we also have $\gamma^1 = ((T_K^1)^{-1},Id)_\#\nu^1$. Let's note $\Tilde{T}_K^1=(T_K^1)^{-1}$. Then, the equalities that we want to show are
    \begin{equation*}
        \begin{cases}
            \iint \xi(x_1,T_K^1(x_1))\phi(x_2)\gamma^{2|1}_x((x_1,T_K^1(x_1)),\mathrm{d}x_2) \mathrm{d}\mu^1(x_1) = \iint \xi(x_1,T_K^1(x_1)) \phi(x_2)\mu^{2|1}(x_1,\mathrm{d}x_2)\mathrm{d}\mu^1(x_1)  \\
            \iint \xi(\Tilde{T}_K^1(y_1),y_1) \psi(y_2)\gamma_y^{2|1}((\Tilde{T}_K^1(y_1),y_1),\mathrm{d}y_2) \mathrm{d}\nu^1(y_1) = \iint \xi(\Tilde{T}_K^1(y_1),y_1) \psi(y_2)\nu^{2|1}(y_1,\mathrm{d}y_2) \mathrm{d}\nu^1(y_1).
        \end{cases}
    \end{equation*}
    And we have indeed
    \begin{align*}
         \iint \xi(x_1,T_K^1(x_1))\phi(x_2)\gamma^{2|1}_x((x_1,T_K^1(x_1)),\mathrm{d}x_2) \mathrm{d}\mu^1(x_1) &= \iint \xi(x_1,T_K^1(x_1))\phi(x_2) \mathrm{d}\gamma^{1,2}((x_1,x_2),(y_1,y_2)) \\
         &= \iint \xi(x_1,T_K^1(x_1))\phi(x_2) \mathrm{d}\gamma_x^{1,2}(x_1,x_2) \\
         &= \iint \xi(x_1,T_K^1(x_1))\phi(x_2) \mu^{2|1}(x_1,\mathrm{d}x_2)\mathrm{d}\mu^1(x_1).
    \end{align*}
    We can do the same for the $\nu$ part by symmetry.
    
    \proofpart{3}{}
    
    Now, we can proceed the same way by induction. Let $\ell\in\{2,\dots,d\}$ and suppose that the result is true in dimension $\ell-1$ (\emph{i.e.} $\gamma^{1:\ell-1}=\pi^{1:\ell-1}_\#\gamma = \gamma_K^{1:\ell-1}$).
    
    For this part of the proof, we rely on \citep{santambrogio2015optimal}[Theorem 2.23]. We can build a measure $\gamma_K^t\in\mathcal{P}(\mathbb{R}^d\times\mathbb{R}^d)$ such that
    \begin{equation} \label{conditions}
        \begin{cases}
            \pi^x_\#\gamma_K^t = \mu \\
            \pi^y_\#\gamma_K^t = \nu \\
            \pi^{1:\ell-1}_\#\gamma_K^t = \eta_{t,\ell}
        \end{cases}
    \end{equation}
    where $\eta_{t,\ell}$ is the optimal transport plan between $\mu^\ell=\pi^{1:\ell-1}_\#\mu$ and $\nu^\ell=\pi^{1+\ell-1}_\#\nu$ for the objective
    \begin{equation*}
        \iint \sum_{k=1}^{\ell-1}\Big( \prod_{i=1}^{k-1} \lambda_t^{(i)}\Big) (x_kx_k'-y_ky_k')^2\ \mathrm{d}\gamma(x,y)\mathrm{d}\gamma(x',y').
    \end{equation*}
    By induction hypothesis, we have $\eta_{t,\ell}\xrightarrow[t\to 0]{\mathcal{D}} \pi^{1:\ell-1}_\#\gamma_K$. To build such a measure, we can first disintegrate $\mu$ and $\nu$
    \begin{equation*}
        \begin{cases}
            \mu = \mu^{1:\ell-1}\otimes \mu^{\ell:d|1:\ell-1} \\
            \nu = \nu^{1:\ell-1}\otimes \nu^{\ell:d|1:\ell-1},
        \end{cases}
    \end{equation*}
    then we pick the Knothe transport $\gamma_K^{\ell:d|1:\ell-1}$ between $\mu^{\ell:d|1:\ell-1}$ and $\nu^{\ell:d|1:\ell-1}$. Thus, by taking $\gamma_K^T = \eta_{t,\ell}\otimes \gamma_K^{\ell:d|1:\ell-1}$, $\gamma_K^T$ satisfies well the conditions (\ref{conditions}).
    
    Hence, we have,
    \begin{equation*}
        \begin{aligned}
            &\iint \sum_{k=1}^{\ell-1}\Big( \prod_{i=1}^{k-1} \lambda_t^{(i)}\Big) (x_kx_k'-y_ky_k')^2\ \mathrm{d}\gamma_K^t(x,y)\mathrm{d}\gamma_K^t(x',y') \\ &= \iint \sum_{k=1}^{\ell-1}\Big( \prod_{i=1}^{k-1} \lambda_t^{(i)}\Big) (x_kx_k'-y_ky_k')^2\ \mathrm{d}\eta_{t,\ell}(x_{1:\ell-1},y_{1:\ell-1})\mathrm{d}\eta_{t,\ell}(x'_{1:\ell-1},y'_{1:\ell-1}) \\
            &\leq \iint \sum_{k=1}^{\ell-1}\Big( \prod_{i=1}^{k-1} \lambda_t^{(i)}\Big) (x_kx_k'-y_ky_k')^2\ \mathrm{d}\gamma_t(x,y)\mathrm{d}\gamma_t(x',y'),
        \end{aligned}
    \end{equation*}
    and therefore
    \begin{equation*}
        \begin{aligned}
            &\iint\sum_{k=1}^{\ell-1}\Big( \prod_{i=1}^{k-1} \lambda_t^{(i)}\Big) (x_kx_k'-y_ky_k')^2\ \mathrm{d}\gamma_K^t(x,y)\mathrm{d}\gamma_K^t(x',y') + \iint \sum_{k=\ell}^d\Big( \prod_{i=1}^{k-1} \lambda_t^{(i)}\Big)\ (x_kx_k'-y_ky_k')^2\ \mathrm{d}\gamma_t(x,y)\mathrm{d}\gamma_t(x',y') \\
            &\le \mathcal{HW}_t^2(\mu,\nu) \\
            &\le \iint \sum_{k=1}^{\ell-1}\Big( \prod_{i=1}^{k-1} \lambda_t^{(i)}\Big) (x_kx_k'-y_ky_k')^2\ \mathrm{d}\gamma_K^t(x,y)\mathrm{d}\gamma_K^t(x',y') + \iint \sum_{k=\ell}^d\Big( \prod_{i=1}^{k-1} \lambda_t^{(i)}\Big)\ (x_kx_k'-y_ky_k')^2\ \mathrm{d}\gamma_K^t(x,y)\mathrm{d}\gamma_K^t(x',y'). \\
        \end{aligned}
    \end{equation*}
    As before, by substracting the first term, and dividing by $\prod_{i=1}^{\ell-1}\lambda_t^{(i)}$, we get
    \begin{equation*}
        \iint (x_\ell x_\ell'-y_\ell y_\ell')^2\mathrm{d}\gamma_t(x,y)\mathrm{d}\gamma_t(x',y') \leq \iint (x_\ell x_\ell'-y_\ell y_\ell')^2\mathrm{d}\gamma_K^t(x,y)\mathrm{d}\gamma_K^t(x',y').
    \end{equation*}
    For the right hand side, using that $\gamma_K^t = \eta_{t,\ell}\otimes \gamma_K^{\ell:d|1:\ell-1}$, we have
    \begin{equation*}
        \begin{aligned}
            &\iint (x_\ell x_\ell'-y_\ell y_\ell')^2\mathrm{d}\gamma_K^t(x,y)\mathrm{d}\gamma_K^t(x',y') \\ &= \iiiint (x_\ell x_\ell'-y_\ell y_\ell')^2\gamma_K^{\ell:d|1:\ell-1}((x_{1:\ell-1},y_{1:\ell-1}),(\mathrm{d}x_{\ell:d},\mathrm{d}y_{\ell:d}))\gamma_K^{\ell:d|1:\ell-1}((x_{1:\ell-1}',y_{1:\ell-1}'),(\mathrm{d}x_{\ell:d}',\mathrm{d}y_{\ell:d}'))\\&\mathrm{d}\eta_{t,\ell}(x_{1:\ell-1},y_{1:\ell-1})\mathrm{d}\eta_{t,\ell}(x_{1:\ell-1}',y_{1:\ell-1}') \\
            &= \iiiint (x_\ell x_\ell'-y_\ell y_\ell')^2\gamma_K^{\ell|1:\ell-1}((x_{1:\ell-1},y_{1:\ell-1}),(\mathrm{d}x_{\ell},\mathrm{d}y_{\ell}))\gamma_K^{\ell|1:\ell-1}((x_{1:\ell-1}',y_{1:\ell-1}'),(\mathrm{d}x_{\ell}',\mathrm{d}y_{\ell}'))\\&\mathrm{d}\eta_{t,\ell}(x_{1:\ell-1},y_{1:\ell-1})\mathrm{d}\eta_{t,\ell}(x_{1:\ell-1}',y_{1:\ell-1}').
        \end{aligned}
    \end{equation*}
    Let's note for $\eta_{t,\ell}$ almost every $(x_{1:\ell-1},y_{1:\ell-1}),(x_{1:\ell-1}',y_{1:\ell-1}')$
    \begin{equation*}
        GW(\mu^{\ell|1:\ell-1}, \nu^{\ell|1:\ell-1})=\iint (x_\ell x_\ell'-y_\ell y_\ell')^2\gamma_K^{\ell|1:\ell-1}((x_{1:\ell-1},y_{1:\ell-1}),(\mathrm{d}x_{\ell},\mathrm{d}y_{\ell}))\gamma_K^{\ell|1:\ell-1}((x_{1:\ell-1}',y_{1:\ell-1}'),(\mathrm{d}x_{\ell}',\mathrm{d}y_{\ell}')),
    \end{equation*}
    then
    \begin{equation*}
        \iint (x_\ell x_\ell'-y_\ell y_\ell')^2\mathrm{d}\gamma_K^t(x,y)\mathrm{d}\gamma_K^t(x',y') = \iint GW(\mu^{\ell|1:\ell-1}, \nu^{\ell|1:\ell-1}) \mathrm{d}\eta_{t,\ell}(x_{1:\ell-1},y_{1:\ell-1})\mathrm{d}\eta_{t,\ell}(x_{1:\ell-1}',y_{1:\ell-1}').
    \end{equation*}
    By Theorem \ref{billingsley}, we have $\eta_{t,\ell}\otimes \eta_{t,\ell}\xrightarrow[t\to 0]{\mathcal{D}} \pi^{1:\ell-1}_\#\gamma_K \otimes \pi^{1:\ell-1}_\#\gamma_K$. So, if $\eta\mapsto \iint GW(\mu^{\ell|1:\ell-1}, \nu^{\ell|1:\ell-1})\mathrm{d}\eta\mathrm{d}\eta$ is continuous over the transport plans between $\mu^{1:\ell-1}$ and $\nu^{1:\ell-1}$, we have
    \begin{equation*}
        \begin{aligned}
            &\iint (x_\ell x_\ell'-y_\ell y_\ell')^2\mathrm{d}\gamma_K^t(x,y)\mathrm{d}\gamma_K^t(x',y') \\
            \ &\xrightarrow[t\to 0]{} \iint GW(\mu^{\ell|1:\ell-1}, \nu^{\ell|1:\ell-1}) \pi^{1:\ell-1}_\#\gamma_K(\mathrm{d}x_{1:\ell-1},\mathrm{d}y_{1:\ell-1}) \pi^{1:\ell-1}_\#\gamma_K(\mathrm{d}x_{1:\ell-1}',\mathrm{d}y_{1:\ell-1}')
        \end{aligned}
    \end{equation*}
    and
    \begin{equation*}
        \begin{aligned}
            &\iint GW(\mu^{\ell|1:\ell-1}, \nu^{\ell|1:\ell-1}) \pi^{1:\ell-1}_\#\gamma_K(\mathrm{d}x_{1:\ell-1},\mathrm{d}y_{1:\ell-1}) \pi^{1:\ell-1}_\#\gamma_K(\mathrm{d}x_{1:\ell-1}',\mathrm{d}y_{1:\ell-1}') \\
            &= \iint (x_\ell x_\ell'-y_\ell y_\ell')^2 \mathrm{d}\gamma_K(x,y)\mathrm{d}\gamma_K(x',y')
        \end{aligned}
    \end{equation*}
    by replacing the true expression of $GW$ and using the disintegration $\gamma_K = (\pi_K^{1:\ell-1})_\#\gamma_K\otimes \gamma_K^{\ell|1:\ell-1}$.
    
    For the continuity, we can apply \citep{santambrogio2015optimal}[Lemma 1.8] (as in the \citep{santambrogio2015optimal}[Corollary 2.24]) with $X=Y=\mathbb{R}^{\ell-1}\times\mathbb{R}^{\ell-1}$, $\Tilde{X}=\Tilde{Y}=\mathcal{P}(\Omega)$ with $\Omega\subset \mathbb{R}^{d-\ell+1}\times \mathbb{R}^{d-\ell+1}$ and $c(a,b)=GW(a,b)$ which can be bounded on compact supports by $\max |c|$. Moreover, we use Theorem \ref{billingsley} and the fact that $\eta_t\otimes\eta_t \xrightarrow[t\to0]{\mathcal{D}} \gamma_K^{1:\ell-1}\otimes \gamma_K^{1:\ell-1}$.
    
    By taking the limit $t\to 0$, we now get
    \begin{equation*}
        \iint (x_\ell x_\ell'-y_\ell y_\ell')^2\mathrm{d}\gamma(x,y)\mathrm{d}\gamma(x',y') \leq \iint (x_\ell x_\ell'-y_\ell y_\ell')^2\mathrm{d}\gamma_K(x,y)\mathrm{d}\gamma_K(x',y').
    \end{equation*}
    
    We can now disintegrate with respect to $\gamma^{1:\ell-1}$ as before. We just need to prove that the marginals coincide which is done by taking for test  functions
    \begin{equation*}
        \begin{cases}
            \xi(x_1,...,x_{\ell-1},y_1,...,y_{\ell-1})\phi(x_\ell) \\
            \xi(x_1,...,x_{\ell-1},y_1,...,y_{\ell-1})\psi(y_\ell)
        \end{cases}
    \end{equation*}
    and using the fact that the measures are concentrated on $y_k=T_K(x_k)$.
    
    \proofpart{4}{} 
    
    Therefore, we have well $\gamma_t\xrightarrow[t\to 0]{\mathcal{D}}\gamma_K$.
    Finally, for the $L^2$ convergence, we have
    \begin{equation*}
        \int \|T_t(x)-T_K(x)\|_2^2\ \mu(\mathrm{d}x) = \int \|y-T_K(x)\|_2^2\ \mathrm{d}\gamma_t(x,y) \to \int \|y-T_K(x)\|_2^2\ \mathrm{d}\gamma_K(x,y) = 0
    \end{equation*}
    as $\gamma_t=(Id\times T_t)_\#\mu$ and $\gamma_K=(Id\times T_K)_\#\mu$. Hence, $T_t\xrightarrow[t\to 0]{L^2}T_K$.
    
\end{proof}

\section{Solving $\mathcal{HW}$ in the discrete setting} \label{appendix:HW}

In this part, we derive formulas to solve numerically $\mathcal{HW}$.

Let $x_1,\dots,x_n\in\mathbb{R}^d$, $y_1,\dots,y_m\in\mathbb{R}^d$, $\alpha\in\Sigma_n$, $\beta\in\Sigma_m$, $p=\sum_{i=1}^n \alpha_i\delta_{x_i}$ and $q=\sum_{j=1}^m\beta_j\delta_{y_j}$ two discrete measures in $\mathbb{R}^d$. The Hadamard Wasserstein problem \eqref{HadamardWasserstein} becomes in the discrete setting
\begin{equation*}
    \begin{aligned}
        \mathcal{HW}^2(p,q) &= \inf_{\gamma\in\Pi(p,q)}\ \sum_{i,j}\sum_{k,\ell} \|x_i\odot x_k - y_j\odot y_\ell\|_2^2\ \gamma_{i,j}\gamma_{k,\ell} \\
        &= \inf_{\gamma\in\Pi(p,q)}\ \mathcal{E}(\gamma)
    \end{aligned}
\end{equation*}
with $\mathcal{E}(\gamma)=\sum_{i,j}\sum_{k,\ell} \|x_i\odot x_k - y_j\odot y_\ell\|_2^2\ \gamma_{i,j}\gamma_{k,\ell}$.
As denoted in \cite{peyre2016gromov}, if we note $\mathcal{L}_{i,j,k,\ell}=\|x_i\odot x_k-y_j\odot y_\ell\|_2^2$, then we have
\begin{equation*}
     \mathcal{E}(\gamma) = \langle \mathcal{L}\otimes \gamma,\gamma\rangle,
\end{equation*}
where $\otimes$ is defined as
\begin{equation*}
    \mathcal{L}\otimes \gamma = \Big( \sum_{k,\ell} \mathcal{L}_{i,j,k,\ell}\gamma_{k,\ell}\Big)_{i,j}\in\mathbb{R}^{n\times m}.
\end{equation*}

\begin{proposition} \label{formulaHW}
    Let $\gamma\in\Pi(p,q)=\{M\in(\mathbb{R}_+)^{n\times m},\ M\mathbb{1}_m=p,\ M^T\mathbb{1}_n=q\}$, where $\mathbb{1}_n = (1,\dots,1)^T \in\mathbb{R}^n$. Let's note $X=(x_i\odot x_k)_{i,k}\in\mathbb{R}^{n\times n\times d}$, $Y=(y_j\odot y_\ell)_{j,\ell}\in\mathbb{R}^{m\times m\times d}$, $X^{(2)}=(\|X_{i,k}\|_2^2)_{i,k}\in \mathbb{R}^{n\times n}$, $Y^{(2)}=(\|Y_{j,l}\|_2^2)_{j,l}\in\mathbb{R}^{m\times m}$ and, $\forall t\in\{1,...,d\},\ X_t=(X_{i,k,t})_{i,k}\in\mathbb{R}^{n\times n}$ and $Y_t=(Y_{j,\ell,t})_{j,\ell}\in\mathbb{R}^{m\times m}$. Then,
    \begin{equation*}
        \mathcal{L}\otimes \gamma = X^{(2)}p\mathbb{1}_m^T+\mathbb{1}_n q^T (Y^{(2)})^T-2\sum_{t=1}^d X_t\gamma Y_t^T.
    \end{equation*}
\end{proposition}

\begin{proof}[Proof of Proposition \ref{formulaHW}] \label{proof_formulaHW}
    First, we can start by writing
    \begin{equation*}
        \begin{aligned}
            \mathcal{L}_{i,j,k,\ell} &= \|x_i\odot x_k-y_j\odot y_\ell\|_2^2 \\
            &= \|X_{i,k}-Y_{j,\ell}\|_2^2 \\
            &= \|X_{i,k}\|_2^2+\|Y_{j,\ell}\|_2^2-2\langle X_{i,k},Y_{j,\ell}\rangle \\
            &= [X^{(2)}]_{i,k}+[Y^{(2)}]_{j,\ell}-2\langle X_{i,k},Y_{j,\ell}\rangle.
        \end{aligned}
    \end{equation*}
    We cannot directly apply proposition 1 from \citep{peyre2016gromov} (as the third term is a scalar product), but by doing the same type of computation, we get
    \begin{equation*}
        \mathcal{L}\otimes \gamma = A+B+C 
    \end{equation*}
    with
    \begin{equation*}
        A_{i,j} = \sum_{k,\ell} [X^{(2)}]_{i,k}\gamma_{k,\ell} = \sum_k [X^{(2)}]_{i,k}\sum_\ell\gamma_{k,\ell} = \sum_k [X^{(2)}]_{i,k}[\gamma\mathbb{1}_m]_{k,1} = [X^{(2)}\gamma\mathbb{1}_m]_{i,1} = [X^{(2)}p]_{i,1}
    \end{equation*}
    \begin{equation*}
        B_{i,j} = \sum_{k,\ell} [Y^{(2)}]_{j,\ell}\gamma_{k,\ell} = \sum_\ell [Y^{(2)}]_{j,\ell}\sum_k \gamma_{k,\ell} = \sum_\ell [Y^{(2)}]_{j,\ell} [\gamma^T\mathbb{1}_n]_{\ell,1} = [Y^{(2)}\gamma^T\mathbb{1}_n]_{j,1} = [Y^{(2)}q]_{j,1}
    \end{equation*}
    and
    \begin{equation*}
        \begin{aligned}
            C_{i,j} = -2\sum_{k,\ell}\langle X_{i,k},Y_{j,\ell}\rangle \gamma_{k,\ell} &= -2\sum_{k,\ell}\sum_{t=1}^d X_{i,k,t}Y_{j,\ell,t}\gamma_{k,\ell} \\ 
            &= -2\sum_{t=1}^d \sum_k [X_t]_{i,k}\sum_\ell [Y_t]_{j,\ell}\gamma_{\ell,k}^T \\
            &= -2 \sum_{t=1}^d \sum_k [X_t]_{i,k}[Y_t\gamma^T]_{j,k} \\
            &= -2\sum_{t=1}^d [X_t(Y_t\gamma^T)^T]_{i,j}.
        \end{aligned}
    \end{equation*}
    Finally, we have
    \begin{equation*}
        \mathcal{L}\otimes \gamma = X^{(2)}p\mathbb{1}_m^T+\mathbb{1}_n q^T (Y^{(2)})^T-2\sum_{t=1}^d X_t\gamma Y_t^T.
    \end{equation*}
\end{proof}

\begin{remark}
    The complexity for computing $\mathcal{L}\otimes \gamma$ is $O(d(n^2m+m^2n))$.
\end{remark}

\begin{remark}
    For the degenerated cost function \eqref{degenerated_cost}, we just need to replace $X$ and $Y$ by $\Tilde{X_t}=A_t^{\frac12}X$ and $\Tilde{Y_t}=A_t^{\frac12}Y$ in the previous proposition.
\end{remark}

To solve this problem numerically, we can use the conditional gradient algorithm (Algorithm 2 in \citep{vayer2019optimal}). This algorithm only requires to compute the gradient 
\begin{equation*}
    \nabla\mathcal{E}(\gamma) = 2(A+B+C) = 2(\mathcal{L}\otimes \gamma)
\end{equation*}
at each step and a classical OT problem. This algorithm is more efficient than solving directly the quadratic problem. Moreover, while it is a non convex problem, it actually converges to a local stationary point \citep{lacoste2016convergence}.

\begin{figure}[H]
    \centering
    \includegraphics[scale=0.45]{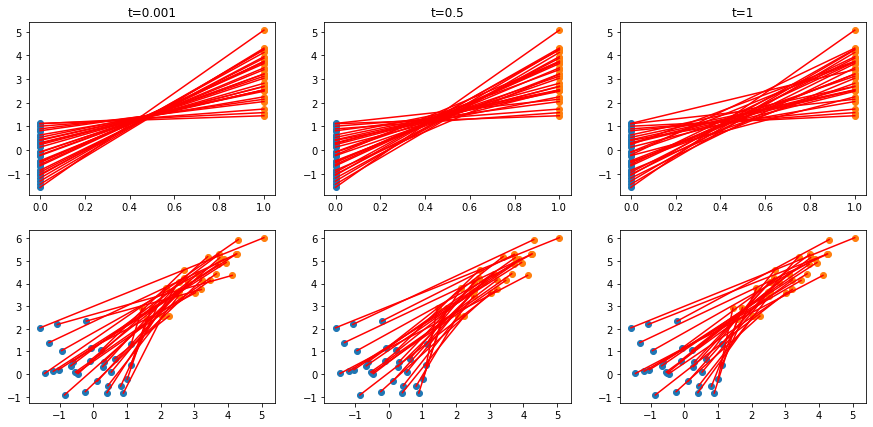}
    \caption{Degenerated Coupling}
    \label{fig:degnerated_coupling}
\end{figure}

On Figure \ref{fig:degnerated_coupling}, we generated 30 points of two gaussian distributions, and computed the optimal coupling of $\mathcal{HW}_t$ for several $t$. These points have the same uniform weight. On the first row, we projected the points on the first coordinate. Note that for discrete points, the Knothe-Rosenblatt coupling comes back to sort the point with respect to the first coordinate if there is no ambiguity (\emph{i.e.} $x_1^{(1)}<\dots<x_n^{(1)}$) as it comes back to perform the optimal transport in one dimension \citep{peyre2019computational}[Remark 2.28]. For our cost, the optimal coupling in 1D can either be the increasing or the decreasing one. We observe well on the first row of figure (\ref{fig:degnerated_coupling}) that the optimal coupling when $t$ is close to 0 corresponds to the ``anti-cdf''.

\end{document}